\documentclass{article} 

\usepackage[preprint]{neurips_2026}

\usepackage{amsmath,amsfonts,amssymb, dsfont, amsthm, mathtools,bm}

\def\eqref#1{equation~\ref{#1}}

\def\1{\bm{1}}

\DeclareMathAlphabet{\mathsfit}{\encodingdefault}{\sfdefault}{m}{sl}
\SetMathAlphabet{\mathsfit}{bold}{\encodingdefault}{\sfdefault}{bx}{n}

\newcommand{\pdata}{p_{\rm{data}}}

\newcommand{\E}{\mathbb{E}}

\newcommand{\KL}{\mathrm{KL}}
\DeclareMathOperator{\TV}{TV}

\DeclareMathOperator*{\argmax}{arg\,max}

\newcommand{\thed}{\mathrm{d}}
\DeclarePairedDelimiter{\norm}{\lVert}{\rVert}
\DeclarePairedDelimiter{\abs}{\lvert}{\rvert}

\newcommand{\cond}{\mathrm{cond}}
\newcommand{\loc}{\mathrm{loc}}
\newcommand{\Reals}{\mathbb{R}}
\newcommand{\ident}{\mathbb{I}}
\newcommand{\normal}{\mathcal{N}}

\newtheorem{hypothesis}{Hypothesis}
\newtheorem{theorem}{Theorem}

\newtheorem{lemma}{Lemma}

\newtheorem{definition}{Definition}

\newtheorem{proposition}{Proposition}
\usepackage{microtype}
\usepackage{import, graphicx, transparent}

\usepackage[pagebackref=true]{hyperref}
\usepackage{url}

\usepackage[dvipsnames]{xcolor}

\makeatletter
\def\thanks#1{\protected@xdef\@thanks{\@thanks
        \protect\footnotetext{#1}}}
\makeatother

\title{Breakdown of Local Denoising as Semantic Speciation}

\author{%
  Guangkuo Liu${^\spadesuit}$\\
  JILA and Department of Physics\\
  University of Colorado Boulder\\
  Boulder, CO 80309, USA \\
  \texttt{guangkuo.liu@colorado.edu} \\
  \And 
  Mert Okyay$^{\spadesuit}$\\
  CTQM and Department of Physics\\
  University of Colorado Boulder\\
  Boulder, CO 80309, USA \\
  \texttt{mert.okyay@colorado.edu} \\
  \And
  Yifan F.~Zhang \\
  Department of Electrical \\
  and Computer Engineering\\
  Princeton University\\
  Princeton, NJ 08544, USA \\
  \texttt{yz4281@princeton.edu} \\
  \And
  Fangjun Hu \\
  QuEra Computing Inc.\\
  1284 Soldiers Field Road, \\
  Boston, MA 02135, USA \\
  \texttt{fhu@quera.com} \\
  \And
  Rahul Nandkishore \\
  CTQM and Department of Physics\\
  University of Colorado Boulder\\
  Boulder, CO 80309, USA \\
  \texttt{rahul.nandkishore@colorado.edu} \\
  \And
  Xun Gao \\
  JILA and Department of Physics\\
  University of Colorado Boulder\\
  Boulder, CO 80309, USA \\
  \texttt{xun.gao@colorado.edu} \\
  \thanks{$\spadesuit$ Equal contribution.}
}
\usepackage{graphicx}

\begin{document}

\maketitle

\begin{abstract}

The dynamics of generative models exhibit two apparently distinct temporal windows: a speciation window, in which a sample commits to a semantic class, and a nonlocality window, in which local context windows become insufficient for generation. Motivated by evidence of their near-concurrence in a variety of frontier models, we investigate their relationship through the spatial distribution of semantic information. Under a ``common cause'' hypothesis, we prove that the nonlocality window must lie in the speciation window.
This hypothesis postulates that semantic labels explain a fraction of the correlations between distant tokens, a condition that is natural for many real datasets. 
We further give conditions under which both windows shrink to a single limiting time as system size grows, defining a ``phase transition'', and verify this behavior analytically in Gaussian mixtures.
Together, these results identify conditions under which semantic information explains the concurrence of speciation and nonlocality, connecting two complementary perspectives on the emergence of semantic structure in generative modeling.

\end{abstract}

\section{Introduction}
\label{sec:intro}
Real-world data often contains distinct semantic classes. An image dataset may contain both cats and dogs; a dataset of English text may include sonnets from Shakespeare as well as essays on history.
A good generative model must be able to sample from the underlying distribution and thus reproduce this structure~\cite{pham2024memorization,shah2025does}. Such features must therefore emerge during inference. 
When during generation is this identity determined, and when can a small perturbation still change it?
We call the period during which the sample commits to a semantic class the \emph{speciation window}~\cite{dynamical_regimes,critical_windows}. 
Early empirical studies show that speciation window persists for a very short time~\cite{SDEedit2022,perceptual_imp2022}, a phenomenon that has been connected to the theory of phase transition in physics~\cite{raya2023spontaneous,dynamical_regimes,hierarchical_diffusion,takahashi2026dynamicalregimesdiscretediffusion}. 

A separate question is how much context a model needs when generating one part of a sample and how this context varies during generation.
For example, generating a patch of fur may require only nearby texture, while generating an animal's eye might require information from farther away to make sure it only has two, and is at the right location.
At some stages, a local neighborhood may be sufficient; at others, restricting the model to that neighborhood may prevent accurate generation.
We call the period during which generation requires information of a neighboorhood size that reaches the size of the image 
\emph{nonlocality window}. Recent work by~\cite{hu2025localdiffusionmodelsphases} has identified this window in simple datasets. Other recent studies analyse the emergence of locality structure through data~\cite{lukoianov2025locality} and how this structure facilitates generalization~\cite{kamb2024analytic,niedoba2024towards,hunt2026exact}, hinting that locality might be a crucial knob to understanding generative modeling.

Although speciation and nonlocality are a priori distinct,
recent experiments suggest that their windows closely align.
A very recent study~\cite{zhang2026concurrencesymmetrybreakingnonlocality} observed this alignment in two open-source diffusion models, DiT-XL~\cite{peebles2023scalable} and Stable Diffusion 3~\cite{esser2024scalingrectifiedflowtransformers}, under an analysis of neural circuitry and controlled experiments.
A complementary analysis of patch-based scores connects architectural locality to collective spatial instabilities and the formation of coherent patterns, with growing spatial correlations observed in trained convolutional diffusion models~\cite{ambrogioni2026outofequilibriumphasetransitionsseed}. Related observations in autoregressive models show that semantic commitment also occurs in short windows beyond diffusion~\cite{li2025blink}.
These observations motivate our central question: why should the period when a sample acquires its semantic identity also be the period when generation needs distant context?

In this work, we formalize this connection theoretically under a simple hypothesis---that semantic information is nonlocally encoded. For example, in a dataset containing zebras and leopards, observing stripes in one region can help predict stripes in a distant region because both reflect the animal’s species. Unconditioned models develop these coherent features across distant domains, necessitating nonlocal computation to coordinate the generation process. As such, the window of semantic speciation must also include a window of nonlocality, temporally aligning the two phenomena.
To make this intuition precise, we identify two information quantities that characterize semantic speciation and nonlocal computation, and establish a quantitative relation between them. \emph{Mutual information} (MI) between local parts of samples and the class label quantifies the amount of the semantic information exposed in local regions, revealing the speciation window~\cite{handke2025measuringsemanticinformationproduction,handke2026entropicsignatureclassspeciation}; \emph{Conditional mutual information} (CMI) between parts of the sample characterizes the error incurred when compute is restricted to a local region~\cite{hu2025localdiffusionmodelsphases}, revealing the nonlocality window. We derive upper and lower bounds on CMI using MI's, under a \emph{common-cause hypothesis}: the semantic labels explain a fraction of the dependence across regions. This formalizes the nonlocal encoding of semantic information as contributions to CMI. These bounds result in the containment of the nonlocality window in the speciation window. Since our theory is fundamentally information-theoretic, the theorem applies to the generation dynamics of both autoregressive and diffusion models, independent of the noising process.
The concrete examples and experiments in this paper focus on diffusion models; direct tests in autoregressive models are left to future work, although recent work \cite{zhu2026dissecting} observes token-entropy spikes at transitions into erroneous reasoning, suggesting a possible link.

We anchor these statements in exact analytical calculations in Gaussian mixture models, where 
score functions and the aforementioned information-theoretic functions are obtained. We also derive the time scales of the semantic and nonlocality windows in this setting. A `thermodynamic limit' (where the data dimension is taken to infinity) closes the semantic and nonlocal windows and yields a sharp \emph{transition}. 
Beyond Gaussian mixtures, we also construct a more general condition for window closure, requiring semantic classes to separate faster than within-class fluctuations as system size grows. Our results unify semantic and architectural perspectives imposed by dynamics along generation trajectories. 

\begin{figure}[t!]
    \centering
    \def\svgwidth{1\textwidth}      
    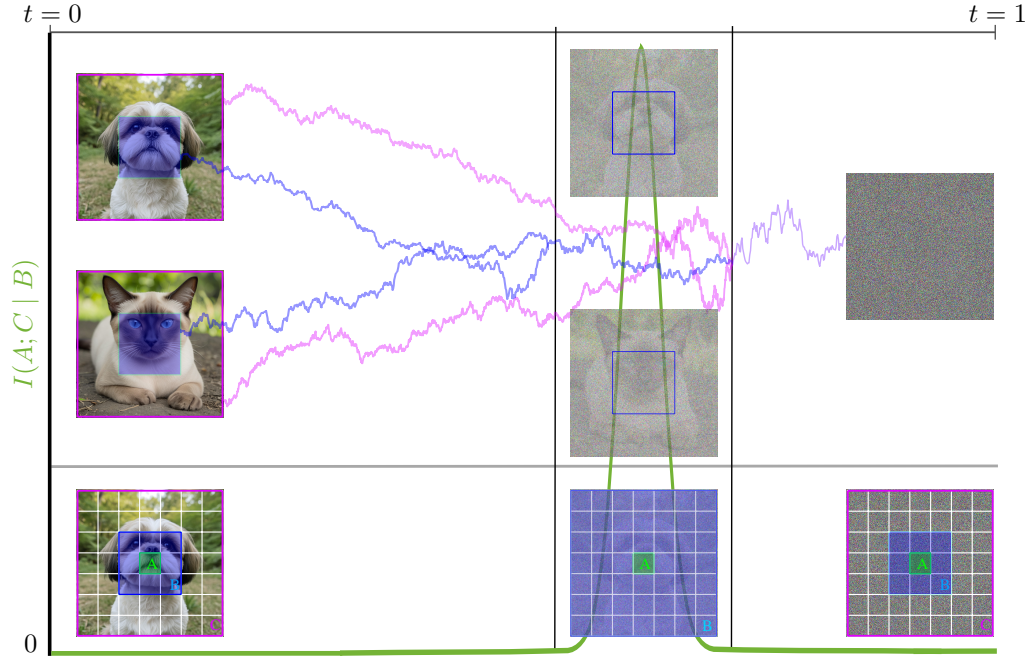
    \caption{\textbf{Semantic speciation locates the window in which local denoising breaks down.} Denoising trajectory is from $t=1$ (right) to $t=0$ (left). The upper panel depicts a semantic speciation from white noise to cat or dog for local region (blue) and global region (magenta); the corresponding thresholds define a speciation window (black vertical lines). The noisy images in the window depict generation times in which the semantic identity is visible in the whole image, but invisible in a local patch (blue box). In the lower row, $A$ is the patch to denoise (green), $B$ its minimal surrounding context (blue) for below-threshold denoising, and $C$ the rest of the image (magenta). The green curve depicts $I_t(A;C\mid B)$ whose peak characterizes a breakdown of exact denoising using local context alone. Under the common-cause hypothesis, Theorem~\ref{thm:window-containment} places the threshold-defined
    nonlocality window inside the speciation window. }
    \label{fig:overview}
\end{figure}

\section{Background and Related Work}

\label{sec:background}

\subsection{Local scores by decaying conditional mutual information}
\label{sec:locality_cmi}

A diffusion model is completely determined by obtaining the score function of the data distribution over diffusion time. If the underlying data distribution is $\pdata$ the score function takes the form $s(x,t) = \nabla_x \log \pdata(x,t)$ where $\pdata(x,t)$ is the data distribution convolved with the Gaussian noise along the diffusion path. We take $X_0 \sim \pdata$ as a sample from the data distribution, and use the interpolation convention with $0 \leq t \leq 1$, where
\begin{equation}
    X_t = (1-t) X_0 + t Z\, , \quad Z \sim \mathcal{N}(0, \ident)\, .
\end{equation}
A brief review of diffusion models is provided in~Appendix~\ref{app:locality_denoising}. 

Our interest is in local diffusion models, in which the generation of a pixel is determined only by its neighbourhood and not need the whole image~\cite{kamb2024analytic,niedoba2024towards,hu2025localdiffusionmodelsphases,hunt2026exact}. If such a task is possible, operationally, we expect the score function to take a local form. We define locality via a tripartition of the image, shown at the left/right of the bottom row of Fig.~\ref{fig:overview}, where the pixel (or small patch) to denoise is denoted $A$, a square annulus around $A$ of radius $r$ is denoted $B$, and the rest of the image is denoted $C$. If $A$ is close to the edges of the image, the parts of the regions outside the image domain is ignored. Images are assumed to live in $\Reals^{d}$, where $d= N\times N$ and $N$ is the linear dimension; we work with square images for simplicity and details of the results do not depend on the aspect ratio. Each image draw $x \in \Reals^{d}$ is then partitioned as $x = (x_A, x_{B}, x_C)$ and as a shorthand, we denote $x_{AB} = (x_A,x_B)$. Thus, implied by the operational constraint, a local score acts on $A$ and depends only on $A$ and $B$ (we use a region $R$ and the pixels $x_R$ in $R$ interchangeably). If the ideal data distribution is known, this would be the score function of the $AB$ marginal, given by $\nabla_A \log p(x_{AB})$ (see~\eqref{eq:local-approximation-gap}). In the absence of the full data distribution, how can such a score be constructed, and when is it useful to do so?

Inspired from work by~\cite{sang_mixed_phase} on mixed-state phases in quantum systems,~\cite{hu2025localdiffusionmodelsphases} bounds the recovery error from using a local denoiser instead of the global one, using the decay length scale of the CMI of the $ABC$ tripartition. The CMI is defined as the remaining mutual information between $A$ and $C$ after revealing the annulus $B$, i.e.,
\begin{equation}
    I(A;C\mid B) = I(A; BC) - I(A; B) \, .
\end{equation}
The CMI, by definition, is the expectation of the Kullback-Leibler (KL) divergence between the joint conditional $p(x_A, x_C \mid x_B)$ and its factorized conditional $p(x_A \mid x_B)p(x_C \mid x_B)$; if the CMI is zero, the joint distribution on ABC  factorizes and the regions form a Markov chain $A - B - C$~\cite{hu2025localdiffusionmodelsphases, zhang2026concurrencesymmetrybreakingnonlocality}. As a result, differentiating the score function on $A$ kills the $C$ dependence completely, where
$\partial_{x_A} \ln p(x) = \partial_{x_A} \ln p_{AB}(x_A, x_B)$ for $p=p_{AB} p_{C|B}$---giving exact locality of the score function. If the CMI is not zero, it still controls the error in denoising a noisy distribution by bounding the total variation. We relegate the careful definition of the quantities in this statement, as well as proving the statement itself (Theorem~\ref{thm:cmi_tv_theorem}) to Appendix~\ref{app:locality_denoising}. 

While the CMI is an information-theoretic quantity that satisfies data-processing inequality in the first two inputs, $A$ and $C$, and thus would decay if only one was noised, it does not satisfy such an inequality on the conditioned variable $B$. Adding noise to $B$ may degrade correlations between $A$ or $C$ and $B$, strengthening the ones between $A$ and $C$, leading to an increase in CMI~\cite{zhang2025conditionalmutualinformationinformationtheoretic}. Along the noise trajectory, CMI may grow too long-ranged while the context window $B$ remains small, such that local denoising incurs larger and larger errors. We call the window in which a local denoiser fails and the global denoiser succeeds the "nonlocality window". 

A problem with CMI as an empirical probe of phenomenology is that information theoretic quantities are notoriously hard to sample in high-dimensional data~\cite{poole2019variationalboundsmutualinformation}. As a consequence, score functions (for diffusion models) are used as operational probes in diagnosing properties of the underlying data structure~\cite{premkumar2026separabilityinformationdiffusionmodels, zhang2026concurrencesymmetrybreakingnonlocality}. While intuitive, is it justified to study the error in the denoiser itself instead of the errors in its outputs? In Appendix~\ref{app:local_scores}, we prove the following bound 
\begin{equation}
\label{eq:locality_gap_cmi}
    \Delta_\loc(t)^4 \leq \frac{4d_A (d_A+3)}{t^4} I_t(A:C\mid B)\, , \quad \Delta_\loc^2(t) = \E_{p_t(x)}\norm{\nabla_A \log p_t(x) - \nabla_A \log p_t(x_{AB})}^2\, , 
\end{equation}
 where $d_A = \dim A$, which implies that the ``locality gap'' $\Delta_\loc$  can be used to probe the CMI. This result complements the results of~\cite{hu2025localdiffusionmodelsphases}: their theorem shows that CMI bounds the error incurred by a local denoiser (see Appendix~\ref{app:locality_denoising} for the careful statement) whereas our result bounds the score matching error in the denoiser \emph{itself}. This establishes that the locality gap itself can serve as a probe of the nonlocality transition, which already was empirically demonstrated by~\cite{zhang2026concurrencesymmetrybreakingnonlocality}.

\subsection{Semantic speciation}
\label{sec:SSB}
Another type of transition is studied through semantic speciation: as forward noise increases, the noisy observation carries less information about the original image’s class. This is commonly probed by forward--backward (FB) experiments, which denoise a corrupted image and assess whether the reconstruction retains its semantic identity~\cite{dynamical_regimes,hierarchical_diffusion,zhang2026concurrencesymmetrybreakingnonlocality}.

An ideal FB experiment draws a clean sample $X_0$ which intrinsically comes with a label $S_0$, adds noise to a region $R$ of the sample to obtain \(X_{R,t}\). Given that observation, it draws a fresh clean image and label from the posterior \(\smash{p(X_{R,0},S\mid X_{R,t})}\). Call the returned label \(\smash{\widehat S}\). Finally, it checks whether \(\smash{\widehat S=S}\). 

Conditioned on \(X_{R,t}\), the original label \(S\) and the returned label \(\widehat S\) are independently drawn from the same distribution
$$ q_s\equiv p_t(s\mid X_{R,t}). $$
Therefore, the label agreement probability for \(X_{R,t}\) is
\begin{align} \Pr(S=\widehat S\mid X_{R,t}) =\sum_s \Pr(S=s\mid X_{R,t}) \Pr(\widehat S=s\mid X_{R,t}) =\sum_s q_s^2. \end{align}
Averaging over the distribution of \(X_{R,t}\) gives the success rate of FB experiment for region \(R\) at time \(t\),
\begin{align}  P_{\mathrm{FB}}(R,t)=\mathbb E\sum_s q_s^2.  \label{eq:fb-rate}\end{align}
Our analysis considers exact posterior sampling and true semantic
labels. If the clean observation determines its label, then
$P_{\mathrm{FB}}(R,0)=1$; at complete noise it approaches
$\sum_s p(s)^2$, equal to $1/L$ for $L$ balanced classes.
Empirical forward--backward experiments, however, need not
reach ideal endpoints because the reverse sampler and classifier
are imperfect. \cite{hierarchical_diffusion}
report a drop in the peak of the source--reconstruction classifier-logit
cosine distribution, while \cite{zhang2026concurrencesymmetrybreakingnonlocality} report a rise in
reconstruction classification error. We also study observations restricted to a spatial region
$R$. Using the same type of classifier-logit cosine probe, our ImageNet~\citep{deng2009imagenet,russakovsky2015imagenet}
crop experiment shows a rapid switch of the empirical peak that occurs
earlier for smaller regions (Appendix~\ref{app:local-fb}).  We call the interval of rapid semantic
change the ``speciation window'', whose starting and ending points are marked by the local and global speciation, respectively, see Section~\ref{sec:spec-window}.

\section{Connecting Speciation and Nonlocality}
\label{sec:PT}

We now connect the nonlocality window, in which local context is
insufficient for accurate denoising, to the speciation window, in
which semantic labels become uncertain under forward noising.
After giving an intuitive argument using score functions, we use an
information-theoretic formulation to establish conditions under
which the nonlocality window is contained in the speciation window.

\subsection{A score-based argument of window containment}\label{sec:score-arg}
Here we provide a heuristic argument using the score function that builds operational intuition for why the nonlocality window is contained in the speciation window.
A natural way to connect the locality gap to semantic information is to subtract the localized version of Bayes' theorem from its global version:
\begin{align}
    \nabla_A \log p_t(x_t) - \nabla_A \log p_t(x_{AB,t})
    ={}&-\big[
    \nabla_A \log p_t(s\mid x_t)
    -\nabla_A \log p_t(s\mid x_{AB,t})
    \big]\notag\\
    &+\big[
    \nabla_A \log p_t(x_t\mid s)
    -\nabla_A \log p_t(x_{AB,t}\mid s)
    \big].
    \label{eq:diff-bayes-rule}
\end{align}
The second bracket is the vector inside the locality gap, conditioned on the
semantic label $s$. In consistency with the common-cause hypothesis~\ref{hyp:common-cause}, we expect an informative label to reduce this
gap by specifying shared aspects of the image's global organization,
leaving less dependence on distant context:
\[
\mathbb{E}\!\left[
\left\|
\nabla_A \log p_t(x_t\mid s)
-\nabla_A \log p_t(x_{AB,t}\mid s)
\right\|_2^2
\right]
<
\mathbb{E}\!\left[
\left\|
\nabla_A \log p_t(x_t)
-\nabla_A \log p_t(x_{AB,t})
\right\|_2^2
\right],
\]
with expectations
taken over noisy images and their semantic labels.

Under this score-gap suppression assumption, a small unconditional
gap implies that both brackets on the right-hand side of
\eqref{eq:diff-bayes-rule} are small in squared expectation.
Conversely, when the unconditional gap is nonzero, the conditional
gap cannot account for it entirely, so the first bracket must also
be nonzero in squared expectation.

To interpret the first bracket, let $s$ be the label of the clean
image. The quantities $p_t(s\mid x_t)$ and $p_t(s\mid x_{AB,t})$ are
the posterior probabilities assigned to that label using global
and local observations, respectively. Their log-gradients measure
how these probabilities respond to perturbations in $A$.
At low noise, both observations can identify the label reliably,
and we expect their semantic posteriors to be relatively insensitive
to small perturbations. At high noise, both observations become
uninformative about the label, and their posteriors approach the
prior. The difference in posterior responses is therefore expected
to be most pronounced between the loss of reliable local label
information and the loss of global label information.
This motivates the conjecture that the nonlocality window lies
within the speciation window. Section~\ref{sec:info-pt} makes this
connection precise using mutual information and an explicit
common-cause assumption.

\subsection{The information-theoretic definition of windows}
\label{sec:info-pt}
\subsubsection{Nonlocality window}
As shown by~\citet{hu2025localdiffusionmodelsphases} (and reproduced in our notation in Appendix~\ref{app:locality_denoising}), the conditional mutual information $I_t(A;C\mid B)$ bounds the error of local recovery. Therefore we define the breakdown window of local denoising as the window where CMI is above a threshold.
\begin{definition}[Nonlocality window]\label{defn:nonlocality_times}
For an annulus tripartition $ABC$, and an information tolerance $\delta>0$, define the start and end of the nonlocality window as
    \begin{align}
    t_{\mathrm{nonloc}}^{\mathrm{start},\delta}
    &:=\inf\{t\in[0,1]:I_t(A;C\mid B)>\delta\},\label{eq:cmi-start}\\
    t_{\mathrm{nonloc}}^{\mathrm{end},\delta}
    &:=\sup\{t\in[0,1]:I_t(A;C\mid B)>\delta\}.\label{eq:cmi-end}
\end{align}
\end{definition}

\subsubsection{Speciation window}\label{sec:spec-window}

We show that the mutual information \(I_t(S; R)\) has the same window where it sharply decreases from the maximal \(H(S)\) to the minimal value \(0\), using the following lemma proved in Appendix~\ref{app:sandwich-proof}.

\begin{lemma}[Two-sided sandwich bounds]\label{lem:sandwich}
    Let $p(s)$ be the prior distribution of labels $s\in S$, $\abs{S} = L$, and write the posterior as $ q_s\equiv p_t(s\mid x_{R,t}) $. Let \( P_{\mathrm{FB}}(R,t)\) denote the success probability of a forward-backward experiment, as described in~\eqref{eq:fb-rate}. On one hand, we have 
    \begin{align}
        \frac12\E\sum_s(q_s-p(s))^2\leq I_t(S; R) \leq \sum_s \E\frac{(q_s-p(s))^2}{p(s)}. \label{eq:sandwich-0}
    \end{align}
    On the other hand, 
    \begin{align}
        \log P_{\mathrm{FB}}(R,t) \ge {I_t(S;R)}-H(S)\ge -h_2\left(P_{\mathrm{FB}}(R,t)\right)-\left(1- P_{\mathrm{FB}}(R,t)\right)\log(L-1) ,  \label{eq:sandwich-1}
    \end{align}
    where $h_2(p)=-p\log{p}-(1-p)\log(1-p)$ is the binary entropy.
\end{lemma}
The first inequality~\ref{eq:sandwich-0} applies to the zero--information endpoint: $I_t(S;R)=0$ if and only if $q_s-p(s)=0$ for all $s$ and $x_{R,t}$, reducing the posterior-sampling FB experiment to prior sampling. The second inequality~\ref{eq:sandwich-1} applies to the full-information endpoint, i.e., $I_t(S;R)=H(S)$ if and only if we succeed at \( P_{\mathrm{FB}}(R,t) =1\). Therefore, \(P_{\mathrm{FB}}(R,t)\) and \(I_t(S;R)\) share the same window where they both drop from their respective maxima to minima. 

The mutual information for a global region $I_t(S;ABC)$ must be greater than that for a local region $I_t(S;B)$ by data processing inequality, indicating that $I_t(S;B)$ has an earlier drop than $I_t(S;ABC)$. Therefore, we choose our definition for speciation window taking this locality nuance into consideration. As the score function arguments in Section~\ref{sec:score-arg} suggest, we define the speciation window starting with nonzero local response to $A$ and ending with global response to $A$.
\begin{definition}[Speciation window]
    For an annulus tripartition $ABC$, and an information tolerance $\delta>0$, define the start and end of the speciation window as
    \begin{align}
        t_{\mathrm{spec}}^{\mathrm{start},\delta}
    &:=\inf\{t\in[0,1]:I_0(S;AB)-I_t(S;B)\geq\delta\},\label{eq:sb-start}\\
    t_{\mathrm{spec}}^{\mathrm{end},\delta}
    &:=\sup\{t\in[0,1]:I_t(S;ABC)\ge\delta\}.\label{eq:sb-end}
    \end{align}
\end{definition}
The start time indicates the earliest time when local label recognition ability drops by a certain amount, and the end time indicates the latest time when global recognition ability retains that amount.

\subsubsection{Common-cause hypothesis}
\begin{figure*}[ht]
\centering
\includegraphics[width=\textwidth]{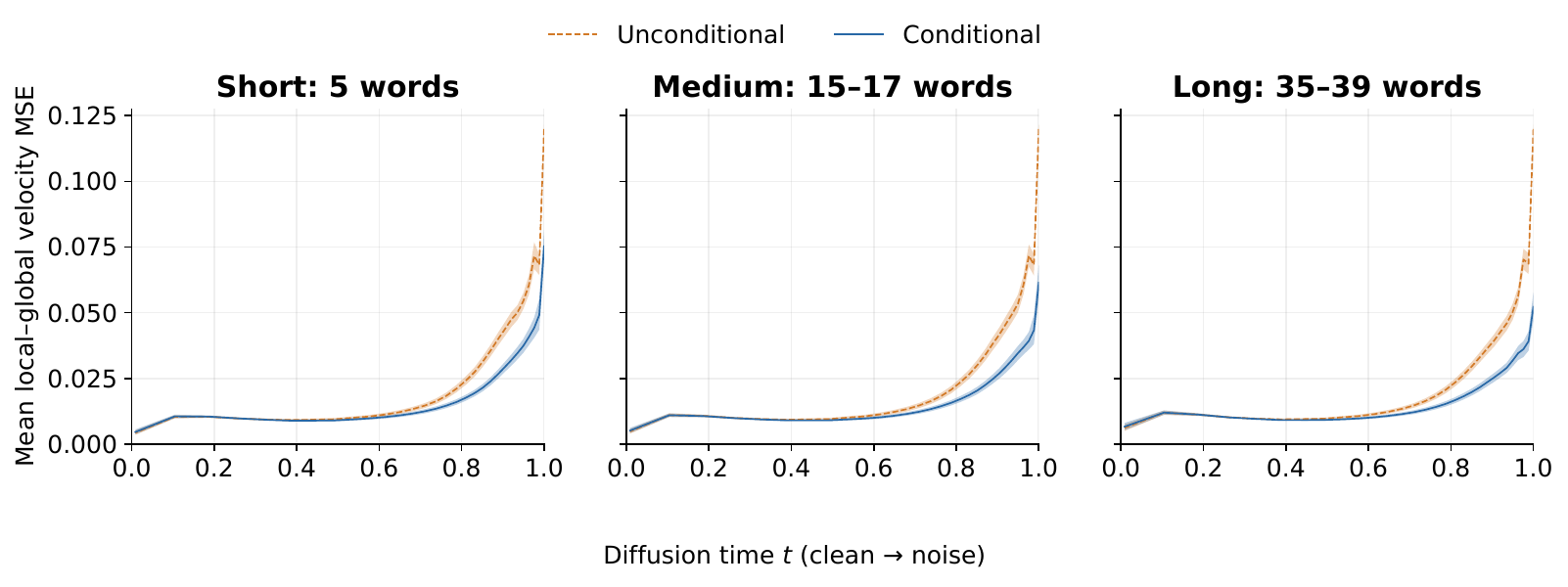}
\caption{
Semantic conditioning reduces the average local--global prediction gap in SD3.
A smaller conditional gap means that restricting distant image-token
interactions changes the model prediction less when semantic information
is supplied. This is consistent with the common-cause interpretation:
the description explains part of the shared image structure that would
otherwise require distant context.
}
\label{fig:common-cause}
\end{figure*}
Conditioning on semantic labels can act in two ways. One is synergistic, where it increases conditional mutual information, \[ I(A;C\mid B,S) \ge I(A;C\mid B). \] 
This occurs when label $S$ contains information that can only be determined by $A$ and $C$ together, but not from any of them alone. For example, count of total objects or parity of a set of bits. The other is redundant, where it decreases conditional mutual information, \[ I(A;C\mid B,S) \le I(A;C\mid B). \]
This occurs when label $S$ acts as a common cause for $A$ and $C$, such as the label of a cat image, which explains the correlation between the cat's head and tail.
We assume that the latter is the case for natural datasets, and we postulate the following \textit{common-cause hypothesis}.
\begin{hypothesis}[Common-cause hypothesis]
    \label{hyp:common-cause}
    For an annulus tripartition $ABC$, semantic labels $S$, and a positive constant $0<\alpha\le 1$, we have
    \begin{align}
        I_t(A;C\mid B,S) \le (1-\alpha) I_t(A;C\mid B)
    \end{align}
    for all time $t\in [0,1]$.
\end{hypothesis}
Such labels always exist: consider the extreme case where the label $S$ specifies the entire clean image, then $I(A;C\mid B,S)=0$, in which case we have $\alpha=1$. In general, we expect that the more informative the label is, the larger $\alpha$ is.  

Gaussian mixtures with identity covariance satisfy this hypothesis for the labels matching the mixture labels at $\alpha=1$ as shown in Appendix~\ref{app:gmm_calculations}. For natural datasets, Figure~\ref{fig:common-cause} provides empirical evidence consistent
with the common-cause interpretation: semantic conditioning reduces
the average local--global prediction gap in SD3, suggesting that
shared semantic information reduces reliance on distant image
context. We average over 64 scenes described using 5, 15--17, and 35--39
words, respectively, giving 192 prompts in total and each prompt uses
three random seeds. We observe a time-weighted gap reduction of 20.7\%, 23.4\%, and 24.4\%
for short, medium, and long descriptions (see Appendix~\ref{app:common-cause-experiment} for the
experimental setup, score parameterization).

\subsubsection{Nonlocality window is contained in speciation window}
The two windows are defined using different information quantities. The common-cause hypothesis connects them.
\begin{theorem}[Containment of the nonlocality window]\label{thm:window-containment}
Fix an annulus tripartition $ABC$. Suppose that, for all $t\in[0,1]$, $S\to X_{AB,0}\to X_{AB,t}$ is a Markov chain and Hypothesis~\ref{hyp:common-cause} holds with the same $\alpha>0$. Choose $0<\delta<\sup_{t\in[0,1]}I_t(A;C\mid B)$. Then
\begin{equation}\label{eq:window-containment}
t_{\mathrm{spec}}^{\mathrm{start},\alpha\delta}
\leq t_{\mathrm{nonloc}}^{\mathrm{start},\delta}
\leq t_{\mathrm{nonloc}}^{\mathrm{end},\delta}
\leq t_{\mathrm{spec}}^{\mathrm{end},\alpha\delta}.
\end{equation}
\end{theorem}

\begin{proof}
The common-cause hypothesis gives, at each time,
\begin{equation}\label{eq:containment-common-cause}
\alpha I_t(A;C\mid B)
\leq I_t(A;C\mid B)-I_t(A;C\mid B,S).
\end{equation}
By the chain rule, the difference on the right has two useful forms. The first is
\begin{align}
I_t(A;C\mid B)-I_t(A;C\mid B,S)
&=I_t(S;AB)-I_t(S;B)-I_t(S;ABC)+I_t(S;BC)\notag\\
&\leq I_0(S;AB)-I_t(S;B).\label{eq:i3-chainrule1}
\end{align}
The last step uses $I_t(S;AB)\leq I_0(S;AB)$, by data processing along the forward noising channel, and $I_t(S;ABC)\geq I_t(S;BC)$, by spatial data processing. Rearranging the same four mutual informations gives the second form,
\begin{align}
I_t(A;C\mid B)-I_t(A;C\mid B,S)
&=I_t(S;BC)-I_t(S;B)-I_t(S;ABC)+I_t(S;AB)\notag\\
&\leq I_t(S;ABC).\label{eq:i3-chainrule2}
\end{align}
Here $I_t(S;B)\geq0$ and $I_t(S;ABC)-I_t(S;AB)\geq0$, while $I_t(S;BC)\leq I_t(S;ABC)$.

Before $\smash{t_{\mathrm{spec}}^{\mathrm{start},\alpha\delta}}$, the definition of the speciation start gives $I_0(S;AB)-I_t(S;B)<\alpha\delta$. Equations~\ref{eq:containment-common-cause} and~\eqref{eq:i3-chainrule1} then give $I_t(A;C\mid B)<\delta$. After $\smash{t_{\mathrm{spec}}^{\mathrm{end},\alpha\delta}}$, the definition of the speciation end gives $I_t(S;ABC)<\alpha\delta$; equations~\ref{eq:containment-common-cause} and~\ref{eq:i3-chainrule2} again give $I_t(A;C\mid B)<\delta$. Thus every time at which CMI exceeds $\delta$ lies between the speciation endpoints. Since $\delta$ is below the CMI supremum, this set is nonempty. Taking its infimum and supremum proves the claim.
\end{proof}

\subsection{System size scaling of two windows and phase transitions}

The containment theorem~\ref{thm:window-containment} relates the two windows, but does not determine how their widths scale with system size. 
This scaling describes whether a finite-size crossover sharpens into a phase transition.
To get a sense of how the windows scale with system size in simple distributions, we explicitly study the windows for a two-component Gaussian mixture. Relegating the calculational details to Appendix~\ref{app:gmm_calculations}, we only quote the analytical results. We study the distribution $p(x,s)$ where $p(x\mid s) \propto \exp(-(x-\mu_s)^2/2\sigma^2)$ and $p(s = \pm 1) = 1/2$, taking $\smash{\mu_s = s \vec{1} \in \Reals^d}$. The scalings for the two windows for this distribution are 
\begin{equation}\label{eq:gmm_times}
1-t_{\mathrm{spec}}^{\mathrm{start},\alpha\delta}
\asymp 
\frac{\sqrt{\log N}}{N}  \,, \;
1-t_{\mathrm{nonloc}}^{\mathrm{start},\delta} 
\asymp
\frac{1}{N}\, ,   \;
1-t_{\mathrm{nonloc}}^{\mathrm{end},\delta}
\asymp
\frac{1}{N}\,,   \;
1-t_{\mathrm{spec}}^{\mathrm{end},\alpha\delta} 
\asymp
\frac{1}{N^2},
\end{equation}
which all converge to $t=1$ in the large-$N$ limit. The intuition behind these scalings follow from the simple fact that decoding success is determined by the signal-to-noise ratio $\norm{\mu_s(t)}/\sigma(t)$. For $\mu_s$ we chose, the signal in $B$ or $BC$ is extensive $\norm{\mu_s} \propto N (1-t)$, and thus decoding becomes ambiguous at $1-t = O(1/N)$. Going left to right (choosing $\delta \sim 1/N^2$ to get a nontrivial CMI) the start of the speciation window is marked by when the classifier using $B$ loses $\delta$ amount of information, whose error is given by a Gaussian tail; setting this equal to the loss yields a $\sqrt{\log 1/\delta}$ in addition to the signal to noise factor $1/N$. The locality windows are about a similar but a different question: when is decoding hard with $B$ only but easy with $BC$? Since the signal is extensive in each, both locality windows directly inherit $1-t = O(1/N)$. Lastly, in the end of the speciation window, we are in the weak signal regime, and even the global classifier is weak. From a quadratic approximation to the mutual information $I(S:ABC)$ integral as $N^2(1-t)^2$, which needs to be equal to $\delta$, we get $1-t = O(1/N^2)$. In summary, the mechanism behind this sharpening can be understood through the decoding transition of the Gaussian mixture model, which behaves like a soft repetition code, where each local patch carries partial information about the semantic label, and a growing buffer combines these clues. 

Inspired by the Gaussian mixture results, we give sufficient conditions for both windows to sharpen to a common critical point at pure noise, $t=1$, as system size grows. The theorem below makes this condition for local semantic identification precise and generalizes the calculation restricted to the Gaussian mixture model.

\begin{theorem}(informal)
\label{thm:separated-means}
    Let $m_s=\mathbb E[X_{B,0}\mid S=s]$ be the mean image on the buffer for label $s$, and write
\begin{equation*}
D_N=\min_{s\ne s'}\|m_s-m_{s'}\|^2.
\end{equation*}
Suppose within-class fluctuations have Gaussian-type tails with variance scale at most $C_N$ along every unit direction.
Thus $D_N$ measures the separation of the class means, while $C_N$ measures the fluctuation that can obscure this separation.
For a fixed finite label set, assume the common-cause hypothesis holds with size-independent $\alpha>0$ and use $X_t=(1-t)X_0+tZ$.
If $D_N/(C_N+1)\to\infty$, then both windows shrink toward $t=1$, provided they exist and their defining information thresholds do not decrease too rapidly with system size relative to this growing separation-to-noise ratio.
The precise regularity condition on the thresholds and the full assumptions are given in Appendix~\ref{app:separated-means}.
\end{theorem}
When $D_N/(C_N+1)$ grows, the semantic signal in the buffer increasingly dominates the within-class fluctuations and the added diffusion noise.
As in a repetition code, the label can then be read from the buffer with vanishing error at any fixed $t<1$; only at pure noise is the signal completely lost.

Reliable decoding makes the remaining label uncertainty $H(S\mid X_{B,t})$ small.
The common-cause hypothesis and the chain rule then give
\begin{equation*}
\alpha I_t(A;C\mid B)
\le I_t(S;A\mid B)-I_t(S;A\mid B,C)
\le H(S\mid X_{B,t}).
\end{equation*}
Thus, when the buffer already identifies the label, the CMI is small.
Moreover, $I_0(S;AB)-I_t(S;B)\le H(S\mid X_{B,t})$, so the speciation window cannot start until the buffer begins to lose the label.
The quantitative decoding bound pushes this start toward $t=1$ at the allowed thresholds, and Theorem~\ref{thm:window-containment} places the other three endpoints between that start and $1$.

\section{Conclusion}

We connect semantic speciation and nonlocality through the information shared between distant regions of a sample through our common-cause hypothesis.
Under this hypothesis, we prove that the nonlocality window lies within the speciation window.
This connection interprets semantics as shared global information and speciation as its decoding during generation.
After demonstrating that both windows close for Gaussian mixture data distributions as system size grows, inspired by its properties, we further give a sufficient condition for both windows to sharpen to a common critical point at pure noise as system size grows.

These results open new questions about how semantic information shapes nonlocality and phase transitions in generative models. The common-cause and semantic-decoding assumptions can be tested in graphical models and real datasets. These studies can reveal how spatial structure affects when the transitions occur and how sharply they develop. The theory can also guide when denoisers use distant context and how their receptive fields change during generation. We leave these directions to future work.

\subsubsection*{Acknowledgments}

We acknowledge assistance from generative AI tools for writing, coding and verifying mathematical claims and proofs.
G.\,L.\, would like to thank Jialiang Zhang and Ruohua Li for discussion on FlexAttention implementation. G.\,L.\,and X.\,G.\,acknowledge support from NSF PFC grant No. PHYS 2317149.
F.\,H.\,acknowledges support from the QuEra Quantum Innovation Postdoctoral Fellowship.

\bibliographystyle{abbrvnat}
\bibliography{refs}

\appendix

\section{Local Denoising Models and Error Bounds}
\label{app:locality_denoising}

In this appendix, we review and advance various aspects of local diffusion models. First, we review a sufficient condition for locality of a denoiser in terms of the conditional mutual information, proven already in~\cite{hu2025localdiffusionmodelsphases}, for the reader's convenience. Second, specialising to Gaussian noise relevant for diffusion models, we prove that the same quantity also bounds the locality gap---a quantity introduced in~\cite{zhang2026concurrencesymmetrybreakingnonlocality} to probe the locality transition. This directly connects the observations therein to the CMI, and supports the notion that locality gap is sensitive to the locality transition. Noting that mutual informations are very hard to sample, training a diffusion model and measuring the locality gap provides a way of controlling CMI, which follows the same spirit as~\cite{yu2025mmgmutualinformationestimation}.

\subsection{Denoising generative models}
Denoising generative models learn to recover clean data from corrupted observations. Let $p, q : \Omega \to \Reals$ be probability densities where $\Omega$ is the space of events. Let $\mathcal{N}(y\mid x)$ be a noise channel. It induces a noisy version of a data density $p$ as 
\begin{equation}
    \mathcal{N}(p)(y) = \int_{\Omega} \thed x \, \mathcal{N}(y \mid x ) p(x)\, .
\end{equation}

$\mathcal{B}_{\mathcal{N}, q}(y\mid x)$ is its Bayes recovery channel with prior $q$, defined as 
\begin{equation}
    \mathcal{B}_{\mathcal{N}, q}(y\mid x) = \frac{\mathcal{N}(y \mid x) q(x)}{\mathcal{N}(q)(y)}\, .
\end{equation}

Generation begins from a tractable, highly corrupted distribution and successively applies learned approximations to such reverse channels. This framework includes autoregressive models, for which the forward channel progressively masks a suffix of the variables and the reverse process reveals \(X_i\) according to \(p(X_i\mid X_{<i})\), as well as diffusion models, for which the corruption is gradual, typically Gaussian, and the reverse dynamics are parameterized through a denoiser or score function. The results below are first stated for a generic noise channel and then specialized to diffusion models.

\subsection{Bounding local recovery error with CMI (generic denoising model)}

In this section, we prove that CMI bounds the squared total variance (TV) between a distribution and its noised and then locally denoised version. We will not try be rigorous in the measure-theoretic sense, but the statements can be formalised as such if desired. 

The nontrivial ingredient in proving the desired result is the classical Fawzi-Renner inequality from~\cite{Fawzi_2015}, which we now state.
\begin{proposition}[Classical Fawzi-Renner Inequality]
    Let $\hat{p}:\Omega \to \Reals$ be the distribution denoised with prior $q$ after applying the noise channel, i.e., 
    \begin{equation*}
        \hat{p}(x) \coloneq [\mathcal{B}_{\mathcal{N}, q} (\mathcal{N}(p))](x)\, .
    \end{equation*}
    then, 
    \begin{equation}
        D_\KL(p \Vert q) - D_\KL(\mathcal{N}(p) \Vert \mathcal{N}(q)) \geq D_\KL(p \Vert \hat{p})\, , 
    \end{equation}
    where $D_\KL$ is the Kullback-Leibler divergence/relative entropy. 
\end{proposition}

Now suppose our samples $x$ can be partitioned into three sets, $A$, $B$, $C$ with respective r.v.s $x_A, x_B$, $x_C$ independent of any spatial geometry. Further suppose the noise channel acts only on $A$, i.e., $\smash{\mathcal{N}(p)(y) = \sum_{x_A \in \Omega_A}\mathcal{N}(y \mid x ) p(x)}$ where, in particular, $y_{B} = x_B$ and $y_C = x_C$. We also denote the clean variables as $X$ and the noisy variables as $Y$. Then, taking $p = p(x)$ and $q(x) = p(x_A, x_B) p(x_C)$, by definition of the mutual information $I(X;Y)$ between two r.v.s $X$ and $Y$ as the relative entropy between the joint and the product distribution, we have 
\begin{equation}
    D_\KL(p \Vert q) = I(AB; C)\, , \quad D_{\KL}(\mathcal{N}(p) \Vert \mathcal{N}(q)) = I(\tilde{A}B; C)\, .
\end{equation}
The CMI is defined in terms of MI's, for our application, we write (for $Z$ a generic r.v.) 
\begin{equation}
    I(Z B ; C)  =  I(Z;C\mid B) + I(B ;C)\, . 
\end{equation}
The term $I(B;C)$ cancels between the two MI's, giving 
\begin{equation*}
    D_\KL(p \Vert q) - D_\KL(\mathcal{N}(p)\Vert \mathcal{N}(q))= I(A;C \mid B) - I(\tilde{A} ; C \mid B) \leq I(A;C \mid B) \, , 
\end{equation*}
by positivity of the CMI $I(\tilde{A} ; C \mid B) \geq 0$. Thus, we have obtained 
\begin{equation}
    I(A;C \mid B) \geq D_{\KL}(p \Vert \hat{p})\, .
\end{equation}
Lastly, by Pinsker's inequality~\cite{mackay2003information}, we have $D_{\KL}(p \Vert \hat{p}) \geq 2\TV(p \Vert \hat{p})^2$ where $\TV(p \Vert q) = \int_\Omega \abs{p(x) - \hat{p}(x)}/2$. Chaining the inequalities together, we obtain the following theorem.
\begin{theorem}[CMI bounds local reconstruction error~\cite{hu2025localdiffusionmodelsphases}.]
\label{thm:cmi_tv_theorem}

For a distribution $p$, a noise channel $\mathcal{N}$ acting only on $A$, with an $ABC$ tripartition and with $\hat{p}(x) =  [\mathcal{B}_{\mathcal{N}, p(x_A, x_B) p(x_C)} (\mathcal{N}(p))](x)$ we have 
\begin{equation}\label{eq:hu_local_denoise}
    2\TV(p \Vert \hat{p})^2 \leq I(A ; C \mid B)\, ,
\end{equation}
i.e., CMI of the distribution before the noise is added controls the locality of the denoiser.
\end{theorem}

The last thing to clear up is to show that the denoiser in Theorem~\ref{thm:cmi_tv_theorem} is truly local. This follows by definition, where $q = p(x_A, x_B) p(x_C)$ and
\begin{equation}
    \mathcal{B}_{\mathcal{N}, q} (\mathcal{N}(p))(x) = \frac{\mathcal{N}(y_A \mid x_A) p(x_A x_B) p(x_C)}{\int \thed x_A \mathcal{N}(y_A \mid x_A) p(x_A x_B) p(x_C)} \\ = \frac{\mathcal{N}(y_A \mid x_A) p(x_A x_B)}{\int \thed x_A \mathcal{N}(y_A \mid x_A) p(x_A x_B)}\, .
\end{equation}

Suppose we break down a single round of diffusion noise pixel-by-pixel, and only consider a single step of this round. Taking $A$ to be the pixel to whom the noise is added, the CMI bounds the total variance (by Pinsker's inequality and Fawzi-Renner, as applied in~\citealp[Section~III.2 and Supplementary Material~S2.A]{hu2025localdiffusionmodelsphases}
\begin{equation}
\label{eq:TV_CMIbound}
    \TV(p\Vert \hat{p})^2 \leq I(A;C\mid B)\, , 
\end{equation}
where $\TV(p\Vert q)$ is the total variation $ \sum_{x \in \Omega} \abs{p(x) - q(x)}/2$ and $\hat{p}(x) \coloneq (\mathcal{B}_{\mathcal{N}, p(x_A,x_B)} (\mathcal{N}_A (p))(x)$ is the denoised distribution after adding noise to $P$ only on $A$ via the noise channel $\mathcal{N}(y_A \mid x_A)$. Lastly, $\mathcal{B}_{\mathcal{N}, \mathcal{N}_A}$ is the ``Bayes recovery channel'' given by Bayes' rule $\mathcal{B}_{\mathcal{N}, q}(x \mid y) = \mathcal{N}(y \mid x)q(x) / \mathcal{N}(q)(y)$ for any probability distribution $q(x)$. If the CMI is zero, the total variance between the initial distribution $p$ and the noised--locally-denoised $\hat{p}$ is zero---the local denoiser is exact. If instead, the CMI is not small but decays exponentially in the radius of $B$, i.e. if we have 
\begin{equation}
    I(A;C \mid B) \leq \gamma \exp(-r/\xi)\, , 
\end{equation}
where $\xi$ is known as the \emph{Markov length}~\citep[Section~III.1]{hu2025localdiffusionmodelsphases},
we can pick a $r$ small enough to guarantee a total variation error of $\epsilon$, i.e., demanding $\epsilon ^2 \geq \gamma \exp(-r/\xi)$, we get 
\begin{equation}
    r \gtrsim \xi \log \frac{\sqrt\gamma}{\epsilon}\, . 
\end{equation}
Then, stitching together many single-step diffusions, we may derive a similar bound on the full diffusion path~\citep[Section~III.3]{hu2025localdiffusionmodelsphases} where each individual step size must at least be $r \geq \xi \log (N K \sqrt{\gamma}/\epsilon)$, where $K$ is the window size and $N$ is the number of discretization steps. As such, a decaying CMI \emph{guarantees a local denoiser}.

\subsection{Diffusion models}
\label{sec:diffusion}

Diffusion models, as a special type of denoising generative models, propose a parametrisation of a data distribution based on the Langevin SDE~\cite{sohl-dickstein-diffusion,DDPM20,scoreSDE2021}
\begin{equation}
   \mathrm{d} X_t = \mu(X_t, t) \mathrm{d} t + \sigma(t) \mathrm{d} \eta\, , 
\end{equation}
where $\eta$ is a Gaussian random variable, drawn independently at each time step. Under suitable assumptions, there exists a corresponding Fokker-Planck equation (FPE) 
\begin{equation}
   \partial_t P = -\partial_x (\mu(x, t) P) + \partial_x (\sigma^2(t) \partial_x P)/2 \, .
\end{equation}
The FPE admits an exact time-reversed form~\cite{scoreSDE2021}. 
Let $Q(x,t) = P(x, T-t)$ with $T$ the reversal time
\begin{equation}
   \partial_t Q = -\partial_x (\tilde{\mu}(x,T-t) Q) + \partial_x (\sigma^2(T-t) \partial_x Q)/2 \, ,
\end{equation}
with $\tilde{\mu}(x,t) = -\mu(x, t)  +  \sigma^2(t) \partial_x \log Q$. This implies the reversed SDE has a modified drift shifted by $- \sigma^2(t) \partial_x \log Q$, which we need to know to be able to reverse each trajectory. The hard part is the so-called \emph{score} function $\partial_x \log Q$, which requires the knowledge of the full distribution $P(x,t)$. The application to generative modeling takes $P(x,0) = \pdata$, such that $P(x, t \to \infty) = \mathcal{N}(\mu(x),\sigma^2(\infty)).$
Since sampling the late-time Gaussian is easy, the difficulty of sampling $\pdata$ is purely relegated to the reversal process. Interestingly, this reveals that we do not need all of $\pdata$, we only need the gradient of its log, the score function. Therefore, we do not care about multiplicative constants in $\pdata$, which are intractable to compute in many cases. To see this, take an energy based model $p_\theta(x) = \exp(-E_\theta(x))/Z_\theta$~\citep[Section~18.4]{Goodfellow-et-al-2016}. The score is 
\begin{equation}
   s_\theta(x) = -\partial_x \log Z_\theta - \partial_x E_\theta(x) \, , 
\end{equation}
and the first term is zero $\log Z_\theta$ does not depend on $x$, and as such, is zero.

\subsection{Bounding local denoiser error with CMI (diffusion model)}
\label{app:local_scores}

While we have provided an information-theoretic guarantee on when a local score may be constructed, we have not operationalized what it means to have a local score function and how to interpret deviations from locality in terms of optimal denoising.

Define the expectation of a function $f(x)$ over the noisy data distribution be $p_t(x)$. We now will argue, in expectation, that the score function of the marginal distribution $p(x_{AB})$ (obtained by marginalizing over $x_C$) is the smallest over all t possible local approximations. Suppose we approximate the score $\nabla_A \log p_t(x)$  by a function $f(x_{AB})$. The error in the approximation satisfies
\begin{multline}\label{eq:local-approximation-gap}
    \E_{p_t}\norm{\nabla_A \log p_t(x) - f(x_{AB})}^2 = \E_{p_t}\norm{\nabla_A \log p_t(x) - \nabla_A \log p_t(x_{AB})}^2 \\
    + \E_{p_t}\norm{\nabla_A \log p_t(x_{AB}) - f(x_{AB})}^2\, , 
\end{multline}
and as such, the best local approximation is the score of the $AB$ marginal. Any other approximation incurs an extra cost $\norm{\nabla_A \log p_t(x_{AB}) - f(x_{AB})}^2$ in expectation. 

The proof follows directly, by adding and subtracting the local score marginal $\nabla_A \log p_t(x_{AB})$, which yields  
\begin{align*}
    \E_{p_t(x)}&\norm{\nabla_A \log p_t(x) - f(x_{AB})}^2\\
    & =\E_{p_t(x)}\norm{\nabla_A \log p_t(x) - \nabla_A \log p_t(x_{AB})}^2
    + \E_{p_t(x)}\norm{\nabla_A \log p_t(x_{AB}) - f(x_{AB})}^2 \\
     &+ 2\E_{p_t(x)}(\nabla_A \log p_t(x) - \nabla_A \log p_t(x_{AB})) \cdot  (\nabla_A \log p_t(x_{AB}) - f(x_{AB})) \, .
\end{align*}
Noting the expectation is over $p_t(x)$, we have 
\begin{multline}
     \E_{p_t(x)}(\nabla_A \log p_t(x) - \nabla_A \log p_t(x_{AB})) \cdot  (\nabla_A \log p_t(x_{AB}) - f(x_{AB})) \\
     =
     \E_{p_t(x_{AB})} [(\nabla_A \log p_t(x_{AB}) - f(x_{AB}))\cdot \E_{p_t(x_C \mid x_{AB})}(\nabla_A \log p_t(x) - \nabla_A \log p_t(x_{AB}))  ]  \, .
\end{multline}
By definition, $\int \thed x_C p_t(x_C\mid x_{AB}) \nabla_A \log p_t(x) = \nabla_A \log p_t(x_{AB})$, and the cross term is zero. 

Thus, the error in any local approximation is the \emph{locality} gap 
\begin{equation}
    \Delta_\loc(t)^2 = \E_{p_t(x)}\norm{\nabla_A \log p_t(x) - \nabla_A \log p_t(x_{AB})}^2\, ,
\end{equation}
which is a fundamental probe of the locality of the true score function: it is zero, if and only if, the true score is local. While the locality gap is a so-called \emph{instantenous} probe---it probes the score at a single time instance and does not determine directly the output `quality'---it nevertheless describes the downsteam error in the generated image: using Tweedie's identity applied to the $AB$ marginal, one can show 

\begin{equation}
    \E[x_A \mid x(t)] - \E[x_A \mid x_{AB}(t)] = \frac{t^2}{1-t}(\nabla_A \log p_t(x) - \nabla_A \log p_t(x_{AB}))\, , 
\end{equation}
i.e., the vector difference in the locality gap is the change in the optimal prediction $x_A$ due to revealing $C$, and also 
\begin{equation}
    \E[\abs{x_A  - \E[x_A \mid x_{AB}(t)]}^2] - \E[\abs{x_A  - \E[x_A \mid x(t)]}^2] = \frac{t^4}{(1-t)^2} \Delta_\loc^2\, , 
\end{equation}
i.e., the locality gap is a direct probe of the extra denoising error on $A$ incurred by hiding $C$ in the optimal denoiser. In this sense, the norms of the instantenous probes control pixelwise differences/errors in the denoised output.

Furthermore, just like the total variation between the data distribution $P$ and its noised-then-locally-denoised version $\hat{P}$, we can also show that CMI controls the size of the locality gap. In particular, it is possible to show 

\begin{proposition}[Controlling the locality gap by CMI]
\label{prop:locality_gap_cmi}
Assume
\[
X(t)=\alpha(t)X(0)+\sigma(t)Z,
\qquad \sigma(t)>0,
\]
where $Z$ is an independent standard Gaussian and $X(0)$ has
finite second moment. Write $m=\dim X_A$, and define the locality gap by
\[
\Delta_{\mathrm{loc}}(t)
:=
\left(
\mathbb E_{p_t}
\left\|
\nabla_A\log p_t(x(t))
-
\nabla_A\log p_t(x_{AB}(t))
\right\|^2
\right)^{1/2},
\]
where $\|\cdot\|$ is the Euclidean norm on the $A$ coordinates. Then, with natural logarithms,
\[
\Delta_{\mathrm{loc}}(t)^2
\le
\frac{2\sqrt{m(m+3)}}{\sigma(t)^2}
\sqrt{I_t(A;C\mid B)}.
\]
\end{proposition}

\begin{proof}
Fix $t$. Tweedie's formula and the Hatsell--Nolte identity
\cite[Eq.~(3) and Proposition~1]{dytso2021derivative}
give
\begin{equation}
\label{eq:locality_gaussian_hessian}
\nabla_A^2\log p_t(x(t))
=
-\frac{I_m}{\sigma(t)^2}
+
\frac{
\operatorname{Cov}\!\left(\alpha(t)X_A(0)\mid X(t)\right)
}{\sigma(t)^4}
=
\frac{
\operatorname{Cov}(Z_A\mid X(t))-I_m
}{\sigma(t)^2}.
\end{equation}
The second equality uses
$\alpha(t)X_A(0)=X_A(t)-\sigma(t)Z_A$ at fixed $X(t)$.

We use two standard moment identities. Here $\|\cdot\|_{\mathrm F}$ denotes the Frobenius norm,
whose square is the sum of squared matrix entrie.
First, the \emph{$L^2$-contraction of conditional expectation}
states that, for any square-integrable scalar random variable
$V$ and side information $S$,
\[
\mathbb E\!\left[
    \left(\mathbb E[V\mid S]\right)^2
\right]
\le
\mathbb E[V^2].
\]
This follows from conditional Jensen's inequality and the
tower property
\cite[Theorem~4.1.11]{durrett2019probability}.

By Wick's theorem for Gaussian moments we have,
\[
\mathbb E\|Z_A\|^4
=
\sum_{i,j=1}^{m}
\mathbb E[Z_{A,i}^2Z_{A,j}^2]
=
\sum_{i,j=1}^{m}(1+2\delta_{ij})
=
m(m+2).
\]

Since a covariance matrix is positive semidefinite,
\[
\|\operatorname{Cov}(Z_A\mid X(t))\|_{\mathrm F}
\le
\operatorname{tr}\operatorname{Cov}(Z_A\mid X(t))
\le
\mathbb E[\|Z_A\|^2\mid X(t)].
\]
Thus, expanding \eqref{eq:locality_gaussian_hessian},
dropping the nonpositive trace term, and applying the two
moment identities above gives
\begin{equation}
\label{eq:locality_hessian_bound}
\begin{aligned}
\mathbb E_{p_t}
\|\nabla_A^2\log p_t(x(t))\|_{\mathrm F}^2
&=
\frac{
m
-2\mathbb E\operatorname{tr}\operatorname{Cov}(Z_A\mid X(t))
+\mathbb E\|\operatorname{Cov}(Z_A\mid X(t))\|_{\mathrm F}^2
}{\sigma(t)^4}
\\
&\le
\frac{
m+
\mathbb E\!\left[
\left(\mathbb E[\|Z_A\|^2\mid X(t)]\right)^2
\right]
}{\sigma(t)^4}
\\
&\le
\frac{m+\mathbb E\|Z_A\|^4}{\sigma(t)^4}
=
\frac{m(m+3)}{\sigma(t)^4}.
\end{aligned}
\end{equation}

Now add a fictitious Gaussian noise of variance $u\ge0$ to $A$ alone.
The resulting $A$ variable has the same distribution as
\[
\alpha(t)X_A(0)+\sqrt{\sigma(t)^2+u}\,Z_A,
\]
while $X_B(t)$ and $X_C(t)$ remain unchanged.
The Gaussian noise $Z_A$ is still independent of
$(X(0),X_B(t),X_C(t))$. Therefore
\eqref{eq:locality_hessian_bound} remains valid with
$\sigma(t)^2$ replaced by $\sigma(t)^2+u$.

The \emph{de Bruijn identity} states that, for a
Gaussian-smoothed density $\rho_u(y)$ evolving according to
\[
\partial_u\rho_u
=
\frac12\sum_{i=1}^{m}\partial_{y_i}^2\rho_u,
\]
the entropy satisfies
\[
\frac{d}{du}
\left[-\int_{\mathbb R^m}\rho_u(y)\log\rho_u(y)\,dy\right]
=
\frac12
\int_{\mathbb R^m}
\rho_u(y)\|\nabla_y\log\rho_u(y)\|^2\,dy
\]
\cite[Lemma~1]{wibisono2018convexity}.
It also holds for conditional entropy when the conditioning
variables are unchanged by the added noise: apply the identity
to each conditional density and average.

Apply this identity to
$I(A;C\mid B)=H(A\mid B)-h(A\mid B,C)$ where $H(A)$ is the differential entropy.
The marginal score is the conditional expectation of the
full score,
\[
\mathbb E\!\left[
\nabla_A\log p_t(X(t))
\,\middle|\,
X_{AB}(t)
\right]
=
\nabla_A\log p_t(x_{AB}(t)).
\]
The orthogonal-projection property of conditional expectation
\cite[Theorem~4.1.15]{durrett2019probability}
therefore gives
\begin{equation}
\label{eq:locality_cmi_first_derivative}
\begin{aligned}
&\left.
\frac{d}{du}
I\!\left(
\alpha(t)X_A(0)+\sqrt{\sigma(t)^2+u}\,Z_A
:X_C(t)\mid X_B(t)
\right)
\right|_{u=0}
\\
&\qquad=
\frac12\mathbb E_{p_t}\!\left[
\|\nabla_A\log p_t(x_{AB}(t))\|^2
-
\|\nabla_A\log p_t(x(t))\|^2
\right]
\\
&\qquad=
-\frac12\Delta_{\mathrm{loc}}(t)^2.
\end{aligned}
\end{equation}

The \emph{Fisher-information dissipation identity} states,
for the same Gaussian heat flow, that
\[
\frac{d}{du}
\int_{\mathbb R^m}
\rho_u(y)\|\nabla_y\log\rho_u(y)\|^2\,dy
=
-\int_{\mathbb R^m}
\rho_u(y)\|\nabla_y^2\log\rho_u(y)\|_{\mathrm F}^2\,dy
\]
\cite[Lemma~1]{wibisono2018convexity}.
Together with de Bruijn's identity, this says that the second
entropy derivative is minus one half of the mean squared
log-density Hessian.

Apply this to the two conditional entropies in the CMI,
we obtain
\begin{equation}
\label{eq:locality_cmi_second_derivative}
\begin{aligned}
&\frac{d^2}{du^2}
I\!\left(
\alpha(t)X_A(0)+\sqrt{\sigma(t)^2+u}\,Z_A
:X_C(t)\mid X_B(t)
\right)
\\
&\quad=
\frac12\mathbb E
\left\|
\nabla_A^2\log p\!\left(
\alpha(t)X_A(0)+\sqrt{\sigma(t)^2+u}\,Z_A
\,\middle|\,
X_B(t),X_C(t)
\right)
\right\|_{\mathrm F}^2
\\
&\qquad-
\underbrace{
\frac12\mathbb E
\left\|
\nabla_A^2\log p\!\left(
\alpha(t)X_A(0)+\sqrt{\sigma(t)^2+u}\,Z_A
\,\middle|\,
X_B(t)
\right)
\right\|_{\mathrm F}^2
}_{\displaystyle \ge0}
\\
&\quad\leq
\frac12\mathbb E
\left\|
\nabla_A^2\log p\!\left(
\alpha(t)X_A(0)+\sqrt{\sigma(t)^2+u}\,Z_A
\,\middle|\,
X_B(t),X_C(t)
\right)
\right\|_{\mathrm F}^2
\\
&\quad=
\frac12\mathbb E
\left\|
\nabla_A^2\log p\!\left(
\alpha(t)X_A(0)+\sqrt{\sigma(t)^2+u}\,Z_A
\, ,
X_B(t),X_C(t)
\right)
\right\|_{\mathrm F}^2
\\
&\quad\le
\frac{m(m+3)}{2(\sigma(t)^2+u)^2}
\le
\frac{m(m+3)}{2\sigma(t)^4},
\end{aligned}
\end{equation}
where we used \eqref{eq:locality_hessian_bound} for the second last inequality.
Integrating
\eqref{eq:locality_cmi_second_derivative} twice gives
Taylor's quadratic upper bound
\cite{boyd2004convex}: if \(f''(u)\le L\) for \(u\ge0\), then
\[
f(u)\le f(0)+u f'(0)+\frac L2u^2.
\]
Using \eqref{eq:locality_cmi_first_derivative} and
nonnegativity of CMI, we obtain
\[
\begin{aligned}
0
&\le
I\!\left(
\alpha(t)X_A(0)+\sqrt{\sigma(t)^2+u}\,Z_A
:X_C(t)\mid X_B(t)
\right)
\\
&\le
I_t(A;C\mid B)
-\frac{u}{2}\Delta_{\mathrm{loc}}(t)^2
+\frac{m(m+3)}{4\sigma(t)^4}u^2.
\end{aligned}
\]
Choose the nonnegative minimizer of this quadratic,
\[
u=
\frac{\sigma(t)^4\Delta_{\mathrm{loc}}(t)^2}{m(m+3)}.
\]
Substitution gives
\[
\frac{\sigma(t)^4\Delta_{\mathrm{loc}}(t)^4}{4m(m+3)}
\le
I_t(A;C\mid B).
\]
Rearranging and taking a square root proves the proposition.
\end{proof}

\subsection{Conditional scores and Classifier-Free Guidance}
\label{sec:CFG}

An important aspect of diffusion models is that they produce images that are faithful to classes of semantic information; a good model trained on cat and dog images produces one animal at a time, it does not produce an amalgamation of the two. How does the model steer towards a specific semantic class? To model this behaviour, we assume the data distribution is a joint distribution $p(x,s)$ between images $x$ and labels $s$ where the diffusion noise only acts on the conditional image distribution, i.e., $p_t(x,s) = p(s) p_t(x\mid s)$. Since the model only has access to the image marginal score, we can write the score as using Bayes' rule
\begin{equation}
    \nabla_A \log p_t(x) = \nabla_A \log p_t(x\mid s) - \nabla_A \log p_t(s \mid x)\, . \label{eq:bayes-rule}
\end{equation}
Rearranging this equation, we obtain the conditioning gap 
\begin{equation}
    \Delta_\mathrm{cond}(t)^2 \coloneq \E\norm{\nabla_A \log p_t(s \mid x)}^2 = \E\norm{\nabla_A \log p_t(x) - \nabla_A \log p_t(x \mid s)}^2\, , 
\end{equation}
whose norm answers two questions: $(i)$ how sensitive is a global classifier to a change in $A$ or $(ii)$ how much does conditioning change the global score function? Just like the locality gap, it also is the fundamental score error in approximating the conditional score by any unconditional function of the global image. Furthermore, by using Tweedie on the label-conditional distribution on the full image $p_t(x \mid s)$, we can show that the conditioning gap controls both the pixelwise difference in the denoised image, as well as the extra squared error in the downstream sample due to hiding/revealing the semantic label $s$. Operationally, the score difference $\nabla_A \log p_t(s \mid x)$ is precisely the term added in diffusion models by classifier guidance~\cite{dhariwal2021diffusionmodelsbeatgans} and the application of Bayes' rule to expand it in terms of the conditional and unconditional scores is the basis of classifier-gree guidance~\cite{ho2022classifierfreediffusionguidance} which drops the need to train an explicit classifier to obtain the gradient $\nabla_A \log p_t(s \mid x)$. Thus, the conditioning gap probes how strongly guidance can change the score on $A$ at each time: a small gap means that even revealing the label provides little additional direction for denoising.

\section{Semantic Speciation for Local Region}\label{app:local-fb}
\begin{figure}[ht    ]
\begin{center}
\includegraphics[width=0.5\textwidth]{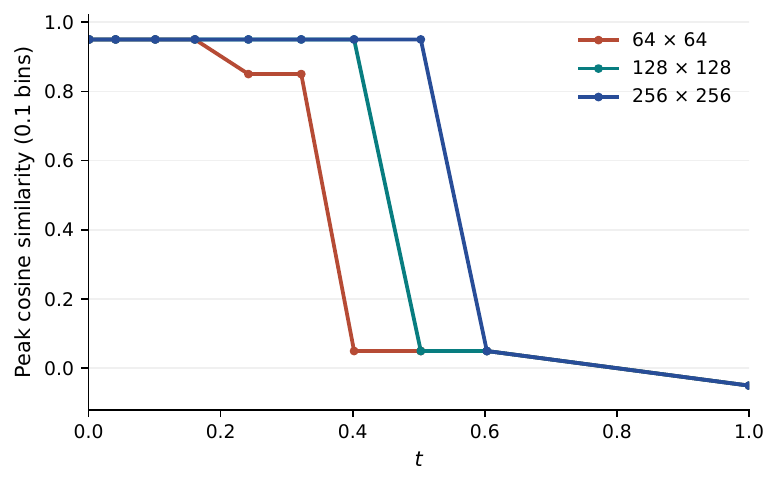}
\end{center}
\caption{Label--reconstruction similarity remains high at early diffusion times
and then falls toward zero. The sharp drop occurs first for the $64\times64$ crop,
then for the $128\times128$ crop, and last for the full image.}
\label{fig:local-fb}
\end{figure}
We adapt the forward--backward ImageNet protocol of
Sclocchi et al.~\cite{hierarchical_diffusion}, using the same unconditional
256-pixel diffusion checkpoint and 250-step respaced reverse sampler. 
Here $t\in[0,1]$ denotes normalized diffusion time. We select one validation
image from each of 100 classes in the ImageNet-1K (ILSVRC2012) validation set. A separate classifier chooses a class-bearing
$64\times64$ crop from 25 candidate locations; a $128\times128$ crop uses the
same center, and the global observation is the full image. Local crops are
enlarged to 256 pixels before noising, so the reverse process receives no
pixels outside the selected region. We draw one reverse sample per image and
time.

Following Sclocchi et al., we measure cosine similarity between classifier
logits of each source and reconstruction, see Figure~\ref{fig:local-fb}. We standardize ConvNeXt V2 Large~\citep{woo2023convnextv2}
logits using 1,000 clean reference images and plot the most populated bin of
the 100 pairwise similarities. The figure uses fixed 0.10-wide bins throughout. 

This binned cosine peak is distinct from
$P_{\mathrm{FB}}(R,t)$; enlarged crops also differ from the full images used
to train the denoiser.

\section{Proof of Lemma~\ref{lem:sandwich}}
\label{app:sandwich-proof}

\begin{proof}
The mutual information is the average divergence between the posterior and the prior:
\begin{equation*}
I_t(S;R)=\mathbb E\sum_s q_s\log\frac{q_s}{p(s)}.
\end{equation*}
For each observation, Pinsker's inequality bounds this divergence below by
$\frac12\mathbb E(\sum_s|q_s-p(s)|)^2$, which is at least
$\frac12\mathbb E\sum_s(q_s-p(s))^2$. The inequality $\log u\leq u-1$
bounds it above by $\sum_s(q_s-p(s))^2/p(s)$.
Averaging proves~\eqref{eq:sandwich-0}.

For~\eqref{eq:sandwich-1}, Jensen's inequality gives
$-\sum_s q_s\log q_s\geq-\log\sum_s q_s^2$.
Average this inequality and use Jensen's inequality together with
$P_{\mathrm{FB}}(R,t)=\mathbb E\sum_s q_s^2$ to get
$H(S\mid X_{R,t})\geq-\log P_{\mathrm{FB}}(R,t)$.
This is the first bound because $I_t(S;R)-H(S)=-H(S\mid X_{R,t})$.

For the other bound, define
$$ E= \begin{cases} 0,&S=\widehat S,\\ 1,&S\ne\widehat S, \end{cases} \qquad e:=\Pr(E=1)=1-P_{\mathrm{FB}}(R,t). $$
The original and returned labels are independent given $X_{R,t}$, so knowing the returned label does not reduce $H(S\mid X_{R,t})$. The entropy chain rule therefore gives
\begin{align}\label{eq:sandwich-chain-rule} H(S\mid X_{R,t}) &=H(S\mid X_{R,t},\widehat S)\notag\\ &=H(S,E\mid X_{R,t},\widehat S)\notag\\ &= \underbrace{H(E\mid X_{R,t},\widehat S)}_{\text{Was the guess wrong?}} + \underbrace{H(S\mid E,X_{R,t},\widehat S)}_{\text{If wrong, which label?}} \end{align}

The uncertainty in this yes-or-no answer, without any additional information, is $ H(E)=h_2(e).$ Knowing \(X_{R,t}\) and \(\widehat S\) can only reduce that uncertainty. Thus
$$ H(E\mid X_{R,t},\widehat S)\le h_2(e). $$
Now suppose we have been told the value of \(E\). When \(E=0\), the original label is exactly \(\widehat S\). There is no remaining uncertainty:
$$ H(S\mid E=0,X_{R,t},\widehat S)=0. $$
When \(E=1\), the original label cannot equal \(\widehat S\), so there are at most \(L-1\) possibilities. A distribution over \(L-1\) possibilities has entropy at most \(\log(L-1)\):
$$ H(S\mid E=1,X_{R,t},\widehat S)\le\log(L-1). $$
The second situation occurs with probability \(e\). Averaging the two cases gives
$$ H(S\mid E,X_{R,t},\widehat S) \le (1-e)\cdot0+e\log(L-1). $$

Substitute into \eqref{eq:sandwich-chain-rule} we obtain the second bound in~\eqref{eq:sandwich-1}.
\end{proof}

\section{Empirical probe of the common-cause hypothesis}
\label{app:common-cause-experiment}
We probe the common-cause interpretation in Stable Diffusion~3
Medium (Figure~\ref{fig:common-cause}).
Each scene has three nested descriptions: longer versions retain
the shorter description and append semantic details.
We generate trajectories at $1024\times1024$ resolution using
30 FlowMatch Euler steps (scheduler shift 3), global attention
throughout, and classifier-free guidance of scale 4.
At every pre-update latent state, we evaluate the same pretrained
weights with global attention and with local attention implemented
using FlexAttention. The local variant restricts image--image
attention to a clipped $15\times15$ token neighborhood
(Chebyshev radius 7) in all 24 transformer blocks, while retaining
all text connections.
For each description-length group, we measure
\begin{equation}
\Delta_b(t)=
\mathbb{E}_{i,r}\!\left[
\frac{1}{d}
\left\|
v_{\theta}^{\mathrm{loc}}(x_t^{i,r},t,c_{i,b})
-
v_{\theta}^{\mathrm{glob}}(x_t^{i,r},t,c_{i,b})
\right\|_2^2
\right],
\qquad b\in\{\mathrm{C},\mathrm{U}\},
\end{equation}
where $d$ is the number of latent coordinates,
$c_{i,\mathrm{C}}$ is the scene description,
$c_{i,\mathrm{U}}=\varnothing$ is the empty prompt,
and the empirical expectation averages scenes $i$ and seeds $r$.
Both branches are evaluated on the same globally guided trajectory
generated for the corresponding description; the predictions
$v_\theta$ are measured before applying guidance.
Because SD3 predicts flow velocity, $\Delta_b$ is a score-gap proxy.

In Figure~\ref{fig:common-cause}, lines average over seeds and then scenes;
shading denotes pointwise 95\% confidence intervals from 20,000
paired scene-cluster bootstrap resamples, stratified by subject category.
The time-weighted reduction
$1-\sum_k w_k \Delta_{\mathrm{cond}}(t_k)/
     \sum_k w_k \Delta_{\mathrm{uncond}}(t_k)$,
with $w_k=t_k-t_{k+1}$, is 20.7\%, 23.4\%, and 24.4\%
for short, medium, and long descriptions.
Diffusion time increases from clean ($t=0$) to noise ($t=1$);
generation proceeds right to left. 

These results provide noise-dependent evidence consistent with
semantic common causes. Attention truncation remains an operational
probe rather than an exact marginal-score construction, so the
experiment does not directly establish the mutual-information
inequality in Hypothesis~\ref{hyp:common-cause}.

\section{Gaussian Mixture Calculations}
\label{app:gmm_calculations}

In this Appendix we provide details of the calculations of the score gaps and the CMI for Gaussian mixtures of various kinds. 

\subsection{Scores for Gaussian Mixtures}

In this subsection we compute various scores and score gaps for Gaussian mixtures, for which almost all results are analytic. 

We consider $N \times N$ images as a random variable $X_0$, whose instances are flattened into vectors $\vec{x}_0 \in \Reals^{N^2}$ and use an interpolation $X_t = (1-t)X_0 + t Z$ where $Z\sim \normal(0, \ident_{N^2})$. Take the joint distribution of labels $S = \pm 1$ and images $X$ to be an equal-weight ferromagnetic Gaussian mixture where $p(S = s) = 1/2$, $p(x_0 \mid S=s) \sim \normal( s \mu \vec{1}, \sigma^2\ident_{N^2})$, where $\vec{1} \in \Reals^{N^2}$ is the vector of ones, $\mu \in \Reals$ is a parameter setting the $O(1)$ separation scale of the means, and $\sigma^2$ is the intrinsic variance of the distribution. The observed distribution is the image marginal 
\begin{equation}
    p(\vec{x}_0) = \frac{1}{Z_{\mu, \sigma^2}}\sum_{s=\pm 1}\left[\exp(-\frac{1}{2 \sigma^2} (\vec{x}- s \mu \vec{1})^2)\right]\, . 
\end{equation}
Since sum of Gaussian random variables remains a Gaussian, the noised distribution itself is a two-component mixture with time-dependent parameters $\mu_t = \mu (1-t)$ and $\sigma^2_t = (1-t)^2 \sigma^2  + t^2$. We thus perform calculations suppressing the time dependence and restore as needed.

We first rewrite the $s$-conditional Gaussian of the marginal image distribution on some region $R$ as 
\begin{equation}
    p(x_R \mid s) = \frac{\alpha(x_R)}{Z_{\mu, \sigma^2}}\exp(s B(x_R))\, , 
\end{equation}
where $B(x_R) = (\mu/\sigma^2)\sum_{i \in R} x_i$ is a soft majority vote of all pixels in region $R$, named $B$ as it is effectively a `magnetic field' bias on region $R$ given label $s$. $s$-independent terms are grouped into the prefactor $\alpha(x_R)$ which will cancel out of relevant scores. Since the label distribution is uniform, this is all we need: the reverse conditional distributions are given by the Bayes' formula corollary 
\begin{equation}
    p(s \mid x_R)= \frac{p(x_R\mid s)}{\sum_{s'=\pm 1} p(x_R \mid s')} = \frac{\exp(s B(x_R))}{2\cosh B(x_R)}\, .
\end{equation}
As a result, Bayes optimal inference of the global label is $\E[S \mid X_R] \equiv \sum_{s\in \pm 1} s p(s \mid x_R) = \tanh B(x_R)$. Thus the task of inferring which Gaussian the observation $x_R$ comes from is equivalent to studying the magnetisation of a non-interacting Ising magnet.

From the two conditional distributions, we can obtain the scores directly. For example,
\begin{equation}
    s_A(x_R\mid s) \coloneq \nabla_{x_A} \log p(x_R \mid s) = -\frac{1}{\sigma^2}(x_A - s \mu)\, ,
\end{equation}
which implies
\begin{equation}
    s_A(x_R) \coloneq \nabla_{x_A} \log p(x_R) = -\frac{x_A}{\sigma^2} + \frac{\mu}{\sigma^2}\tanh B(x_R)\, ,
\end{equation}
via the trick $\nabla_{x_A} \log p(x_R) = \sum_{s}p(s\mid x_R) \nabla_A \log p(x_R \mid s)$. In the score $s_A$, the first term only includes $A$, and the rest of the image enters through the $\tanh$. Then we have the locality gap 
\begin{equation}
    \Delta^U_\loc = \frac{\mu}{\sigma^2}[\tanh B(x) - \tanh B(x_{AB})]\, .
\end{equation}
This gives an explicit characterisation of what the locality gap is comparing: it asks how important $\sum_{i\in C} x_i$ is in inferring the label, softened through the $\tanh$. Its peak in $t$ is directly given by the shift induced by $C$.

The same results also let us calculate the global conditioning gap 
\begin{equation}
    \Delta^G_\cond = \frac{\mu}{\sigma^2}(\tanh(B(x))-s)\, .
\end{equation}
If we look at their difference, 
\begin{equation}
    \Delta^U_\loc - \Delta^G_\cond = \frac{\mu}{\sigma^2}(s - \tanh(B(x_{AB}))\, .
\end{equation}

\subsection{CMI for Gaussian mixtures}

Just like the scores, the conditional mutual information for the two component Gaussian mixture is analytically reducible to a single integral.

\paragraph{Exact formula for CMI.}To start, we recall the definition of the CMI 
\begin{align}
    I(A;C\mid B) 
    &\coloneq \E_{x_B} D_\KL(p(x_A,x_C \mid x_B) \Vert p(x_A \mid x_B) p(x_C\mid x_B) ) \, ,  \\ 
    &\equiv \int \thed x p(x) \log \frac{p(x_A, x_C \mid x_B)}{p(x_A \mid x_B) p(x_C\mid x_B)}\, , 
\end{align}
and get rid of conditioning terms by restoring $B$ marginals. The result is 
\begin{align}
    \label{eq:cmi_entropy_mi_expand}
    I(A;C\mid B) 
    &\equiv \int \thed x p(x) \log \frac{p(x_A, x_C, x_B)p(x_B)}{p(x_A,x_B) p(x_C, x_B)} \\ 
    &= H(AB) -H(B) + H(BC)-H(ABC)\,, \\
    & = I(A; C | B, S) + [I(S; AB) - I(S; B)]- [I(S; ABC) - I(S; BC) ]\, .
\end{align}
For the joint Gaussian mixture, $I(A; C | B, S)=0$ since conditional on the label, the distribution completely factorizes, and we only need to compute mutual informations/differential entropies associated to a region $R$ where $R \in \{AB, BC, B, ABC\}$. 

The entropy $H(R)$ is
\begin{multline}
    H(R) = -\int \thed x_R p(x_R) \log p(x_R) = \frac{1}{2 \sigma^2}\int \thed x_R p(x_R)  (x_R- \mu_R)^2 \\- \int \thed x_R p(x_R) \log (1+ \exp(- \frac{2}{\sigma^2} x_R \cdot \mu_R)\,) , 
\end{multline}
obtained by factoring out one of the Gaussians inside the $\log$. The first term gives two trivial gaussian integrals by expanding $p(x_R)$ as a sum of two Gaussian integrals (which we do not write out, they will cancel out between all the terms in the CMI~\eqref{eq:cmi_entropy_mi_expand}). The second term is the nontrivial one because it includes a sum inside the $\log$, as well as integrals over all the pixels in $R$. However, since the integrand only depends on the dot product $x_R \cdot \mu_R$, we may write $x_R = x_R^\parallel + x_R^\perp$ where $x_R^\perp \cdot \mu_R = 0$ and $\thed x_R = \thed x_R^\parallel \thed x_R^\perp$, and $\thed x_R^\perp$ integrates out to give a constant that cancels with some of the normalisation in $p(x_R)$. The result is a one-dimensional integral 
\begin{equation}
    H(R) = W(R) - \frac{1}{Z_R^\parallel}\int \thed x_R^\parallel \exp\left(-\frac{(x_R^\parallel-\norm{\mu_R})^2}{2\sigma^2}\right) \log \left (1+ \exp(-\frac{2}{\sigma^2} \norm{\mu_R} x_R^\parallel)\right)\, .
\end{equation}
Lastly, we perform a change of variable by writing the parallel component as `mean + fluctuations', i.e. $x_R^\parallel = \mu_R + \sigma z$, which gives  
\begin{equation}
    H(R) = W(R) - \frac{1}{\sqrt{2\pi}}\int \thed z \exp\left(-z^2/2\right) \log \left (1+ \exp(-2 K^2- 2K z)\right)\, ,
\end{equation} 
in terms of one free parameter, the signal-to-noise ratio $K_R \coloneq \norm{\mu_R}/\sigma$. Defining 
\begin{equation}\label{eq:F_integral}
    F(K) \coloneq \frac{1}{\sqrt{2\pi}}\int \thed z \exp\left(-z^2/2\right) \log \left (1+ \exp(-2 K^2- 2K z)\right)\, , 
\end{equation}
we can write the CMI as 
\begin{align}
    I(A;C\mid B) 
    = F(K_B) - F(K_{AB}) - [ F(K_{BC}) - F(K_{ABC})]\, .
\end{align}

\paragraph{Asymptotic expansion for the $F$ integral.}The integral $F(K)$ can be exactly evaluated in the large signal limit $K \gg 1$. Performing another change of variables $ y = 2 K(K+ z)$ yields 
\begin{equation}
    F(K) = \frac{1}{2K}\exp(-K^2/2)\int \frac{\thed y}{\sqrt{2\pi}} \exp\left(-y^2/8K^2\right)\exp(y/2) \log \left (1+ \exp(-y)\right)\, , 
\end{equation}
and as $K \to \infty$, $\exp\left(-y^2/8K^2\right) \to 1$, and 
\begin{equation}
    F(K) \asymp \sqrt\frac{\pi}{2K^2}\exp(-K^2/2)\, , 
\end{equation}
where $\int \thed y \exp(y/2) \log \left (1+ \exp(-y)\right) = 2\pi$ by elementary methods. 

\paragraph{Diverging Markov Length.}Now send the image size to infinity, growing both $|B|$ and $|C|$ as $O(N^2)$ while $|A| = O(1)$, a ``thermodynamic limit'' relevant for denoising a small patch. This yields, in the $K \to \infty$ limit, a simple form for the CMI which reads 
\begin{equation}
    I(A;C \mid B) \asymp \sqrt{\frac{\pi}{2 K_B^2}} \exp\left[-\frac{\mu^2}{2 \sigma^2} \abs{B}\right]\left(1- \exp\left[-\frac{\mu^2}{2 \sigma^2} \abs{A}\right]\right)\, , 
\end{equation}
up to exponentially small corrections in $|C|$. Since this decays faster than exponential in $r \propto \sqrt{B}$, the Markov length $\xi = 0$, and there is no locality transition whenever $K$ is large. 

If we recall the noise path $X_t = (1-t)X_0 + t Z$ with $Z$ standard Gaussian noise, we have a Gaussian mixture for every $t=1$ with effective parameters $\mu_t = \mu (1-t)$ and $\sigma^2_t = (1-t)^2 \sigma^2  + t^2$. As the noise is i.i.d. on every patch, the signal to noise for any $R$ is 
\begin{equation}
    K_R^2(t) = \frac{(1-t)^2\mu^2}{(1-t)^2 \sigma^2 + t^2} |R|\, .
\end{equation}
Is the signal-to-noise large as $N \to \infty$ for every $0\leq t \leq 1$? Taking $ t = 1 - k/N$ yields 
\begin{equation}
    K_R^2(t) \asymp \mu^2 k^2 |R|/N^2 = O(1)\, ,
\end{equation}
for $R= B$ and $R= BC$, and the asymptotic expansion fails. In fact, in this limit, we are able to show Markov length $\xi \to \infty$. Noticing that $K_{R_1 R_2}^2 = K_{R_1}^2 + K_{R_2}^2$ for disjoint $R_1$ and $R_2$, we can write 
\begin{equation}
    F(K_R(t)) - F(K_{AR}(t)) = -K_A(t)^2 F'(K_R(t))\, ,
\end{equation} 
where we use $K_A \to 0$ and the $'$ denotes differentiation with respect to $K^2$. The CMI takes the form 
\begin{equation}
    I_t(A;C \mid B) \asymp K_A^2(t)[F'(K_{BC}(t))-F'(K_B(t))]\, ,
\end{equation}
which we try to maximize over $t$. In particular, setting 
\begin{equation}
    I_\mathrm{peak}= \sup_t I_t(A;C \mid B)\, , \quad  t_* =  \argmax_t I_t(A;C \mid B)
\end{equation}
we are able to show no exponential decay exists. First, define an effective variance parameter 
\begin{equation}
    \sigma_\mathrm{eff}^2(t) = \sigma^2 + \frac{t^2}{(1-t)^2}\, , 
\end{equation}
such that we may interpret the noise path as purely expanding the variances. Then, plug in the definition of $K_R(t)$, which gives  
\begin{equation}
    \sup_t K_A^2(t)[F'(K_{BC}(t))-F'(K_B(t))] = \sup_t \frac{\mu^2|A|}{\sigma_\mathrm{eff}^2(t)}\left[F'\left(\frac{\mu^2(|B|+|C|)}{\sigma_\mathrm{eff}^2(t)}\right)-F'\left(\frac{\mu^2 |B|}{\sigma_\mathrm{eff}^2(t)}\right)\right]\, .
\end{equation}
The limit keeps $|B|/|C| = O(1)$, so let's isolate that parameter, setting $\lambda = 1+|C|/{|B|}$, we get 
\begin{equation}
    \sup_t K_A^2(t)[F'(K_{BC}(t))-F'(K_B(t))] = \frac{|A|}{|B|}\sup_{K_B}K_B^2 \left[F'\left(\lambda K_B^2 \right)-F'\left(K_B^2\right)\right]\, ,
\end{equation}
where we also set $\mu^2|B|/\sigma_\mathrm{eff}^2(t) = K_B^2$. Also note that the supremum over $t$ is equivalent to supremum over $K_B^2$, which is the only $t$ dependent parameter, which is the main point of this substitution. Crucially, the supremum over $K_B^2$ yields just a number and does not depend on $|B|$ by itself anywhere. As a result, $I_{t_*}(A;C\mid B) \asymp 1/|B|$, and the Markov length is divergent. 

Lastly, we can obtain $k$ explicitly. Noting that we need $K_B^2(t_*) \asymp O(1)$ at the supremum, we need to obtain 
\begin{equation}
    \frac{\mu^2 |B|}{\sigma_\mathrm{eff}^2(t_*)} = O(1)\, \implies 1-t_* = \frac{1}{\mu} \sqrt{\frac{K_B^2(t_*)}{|B|}}\, .
\end{equation}
Since $|B| = O(N^2)$, the peak time $t_*$ scales as $1-1/N$, as expected. 

\paragraph{The time windows.}Since we can control exactly the CMI in the Gaussian mixture, let us try to give asymptotic formulas for each, demonstrating that the windows need not be exactly identical.

First, recall the nonlocality times $t^{\mathrm{start/end}, \delta}_\mathrm{nonloc}$ are defined by finding times such that the CMI $I_t(A;C\mid B) \approx \delta$ by Defn.~\ref{defn:nonlocality_times}. We assume we pick $\delta$ as a fixed fraction of the peak CMI value. Since we already evaluated the CMI in the transition window, we immediately know scale identically as the peak, i.e. $t^{\mathrm{start/end}, \delta}_\mathrm{nonloc} \sim 1-O(1/N)$. 

Second, recall that $t^{\mathrm{start/end}, \delta}_\mathrm{spec}$ is defined \emph{mutual informations}. For the start, we want to find the first $t$ such that $I_0(S;AB) - I_t(S;B) = \delta$. Assuming $\delta$ is smaller than the peak value, this is equivalent to calculating $F(K_B(t)) - F(K_{AB}(0)) = \delta$. Since $B$ occupies a large fraction of the image, the second term is exponentially suppressed in $|B|$ and we may drop it. For large signal to noise, we then just need to solve 
\begin{multline}
    \frac{\exp(-K_B^2(t^{\mathrm{start}, \delta}_\mathrm{spec})/2)}{K_B} = \delta \implies K_B^2(t^{\mathrm{start}, \delta}_\mathrm{spec}) = 2 \log \frac{1}{\delta} + O(\log \frac{1}{K_B}) \\\implies 1-t^{\mathrm{start}, \delta}_\mathrm{spec} = \frac{1}{\mu} \sqrt{\frac{K_B^2(t^{\mathrm{start}, \delta}_\mathrm{spec})}{|B|}}\sim \sqrt\frac{\log 1/\delta}{|B|}\, .
\end{multline}
Focusing near $t=1$, near the critical time $\delta \sim 1/N^2$, we thus get $1-t^{\mathrm{start}, \delta}_\mathrm{spec} \sim  \sqrt{\log (N)}/N$. 

Lastly, the speciation end time is defined as the largest time such that $I_t(S;ABC) = \delta)$. Near $t \to 1$, when the global label is degraded, the signal to noise is very small, and instead we evaluate the the integral in~\eqref{eq:F_integral} in the small $K$ limit. This yields $I(S;ABC) = \log 2 - F(K_{ABC}) = K_{ABC}^2/2$ which we set equal to $\delta$. Using the formula for the exponent being $O(1)$, we have 
\begin{equation}
    1-t^{\mathrm{end}, \delta}_\mathrm{spec} = \frac{\sqrt{\delta}}{\sqrt{|ABC|}} = O(1/N^2)\, , 
\end{equation}
where we again took $\delta = O(1/N^2)$. This proves that the locality window has a size proportional to $1/N$, while the semantic window has the scalings $1-O(\sqrt{\log N}/N)$ and $1-O(1/N^2)$.

\begin{figure*}[t]
    \centering

    \begin{minipage}[t]{0.48\textwidth}
          \centering
        \includegraphics[width=\linewidth]{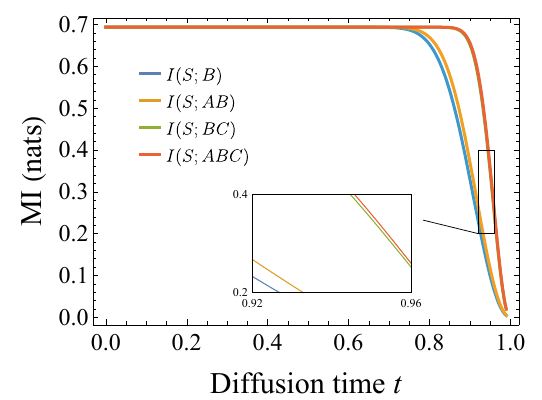}
        \vspace{-0.5em}
        \centerline{\textbf{(a)}}
    \end{minipage}
    \begin{minipage}[t]{0.48\textwidth}
        \centering
        \includegraphics[width=\linewidth]{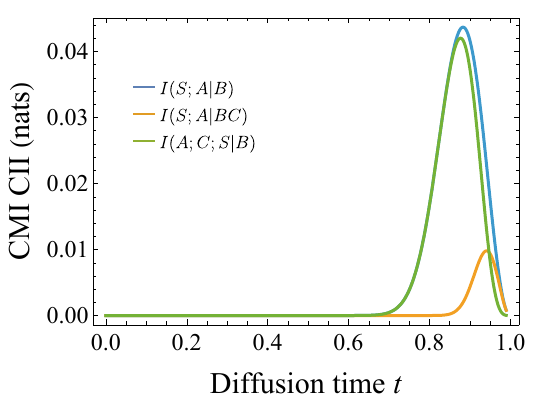}
        \vspace{-0.5em}
        \centerline{\textbf{(b)}}
    \end{minipage}

    \caption{
    Exact numerical evaluation of the Gaussian-mixture information measures with specific values given in \ref{app:numerical-gmm}.
    \textbf{(a)} Mutual information curves for the four spatial regions. The cliff diagnoses the symmetry-breaking phase transition. Inset: zoomed-in region near critical window. 
    \textbf{(b)}   Local and global conditional information gaps $I_t(S;A|B),I_t(S;A|BC)$, and their difference which gives conditional interaction information $I_t(A;C;S|B)$. The peak
    diagnoses the non-locality transition.
    }
    \label{fig:gmm-numerics}
\end{figure*}

\subsection{Numerical evaluation of the exact GMM formulas}
\label{app:numerical-gmm}

We numerically evaluate the exact information-theoretic formulas derived above for the
ultra-local two-component Gaussian mixture model.  We take the homogeneous mean
vector \(\mu=(1,\ldots,1)\), unit bare variance \(\sigma^2=1\), and a \(20\times 20\)
flattened system.  The spatial partition is chosen as
\[
    a=\|\mu_A\|^2=16,\qquad
    b=\|\mu_B\|^2=80,\qquad
    c=\|\mu_C\|^2=304,
\]
corresponding to a local patch \(A\), a buffer \(B\) occupying \(20\%\) of the
system, and the remaining context \(C\).  The forward noising process is
\[
    X_t=(1-t)X_0+tZ,\qquad Z\sim\mathcal{N}(0,I).
\]
We numerically evaluate the core integral \(F(W)\) by adaptive numerical quadrature and the
resulting curves are interpolated on a uniform grid \(t\in[0,0.99]\) with
spacing \(\Delta t=0.005\).

Fig.~\ref{fig:gmm-numerics}(a) tracks the four mutual informations
\(
    I_t(S;B), I_t(S;AB), I_t(S;BC), I_t(S;ABC).
\)
These curves are monotone decreasing under the forward diffusion dynamics and
develop sharp late-time cliffs.  The smaller regions lose semantic information
earlier, while the global regions remain informative until later times.  This
is the finite-size numerical signature of the forward-backward semantic
transition.

Fig.~\ref{fig:gmm-numerics}(b) plots the two conditional gaps
\(
    I_t(S;A|B)=I_t(S;AB)-I_t(S;B),
\)
and
\(
    I_t(S;A|BC)=I_t(S;ABC)-I_t(S;BC).
\)
These gaps peak near the steepest parts of the corresponding MI cliffs.  For
the parameters used here, \(G_{\rm loc}\) peaks near \(t\simeq 0.88\), whereas
\(G_{\rm glob}\) is smaller and delayed, peaking near \(t\simeq 0.94\).  This
delay reflects the greater robustness of the global context \(ABC\) to the
forward noise.

The conditional interaction information
\(
    I_t(A;S;C|B)
    =
    I_t(S;A|B)-I_t(S;A|BC)
\) is also illustrated in Fig.~\ref{fig:gmm-numerics}(b).
For the identity-covariance ultra-local GMM, conditioning on the semantic label
removes all residual spatial dependence, so
\(
    I_t(A;S;C|B)=I_t(A;C|B).
\)
Thus the same curve is also the unconditional CMI diagnosing non-locality.  The
CII/CMI curve has a single sharp peak, numerically located near
\(t\simeq 0.88\) for the present parameters, coinciding with the local semantic
information cliff.  This provides an exact-calculation check of the proposed
concurrence between the symmetry-breaking and non-locality transitions.

\section{Convergence from Separated Class Means}
\label{app:separated-means}

We state and prove the formal version of the convergence result discussed in the main text.

\begin{theorem}[Convergence from separated class means] \label{thm:separated-means-app}
(Formal version of Theorem \ref{thm:separated-means})Let $S$ have a fixed finite set of labels with positive prior probabilities. Suppose that the common-cause hypothesis holds with a size-independent $\alpha>0$. For each $N$, let $B$ be annulus with radius $N$, and use $X_t=(1-t)X_0+tZ$, where $Z\sim\mathcal N(0,I)$ is independent of $(X_0,S)$. Write $m_s=\mathbb E[X_{B,0}\mid S=s]$. Suppose there are constants $D_N,C_N>0$, potentially scaling with $N$, such that, for all $s\ne s'$ and all $u\in\mathbb R^{|B|}$,
\begin{align}
\|m_s-m_{s'}\|_2^2&\geq D_N,\label{eq:mean-separation}\\
\mathbb E\!\left[e^{u^\top(X_{B,0}-m_s)}\mid S=s\right]
&\leq e^{C_N\|u\|_2^2/2}.\label{eq:class-tail-bound}
\end{align}
Suppose the separation between semantic means $D_N$ scales faster than the tail bound within class $C_N$ so that they admit a choice of $\delta_N$ below the CMI peak at each size, $0<\alpha\delta_N<H(S)/2$ and
\begin{equation}\label{eq:concurrence-threshold-rate}
\log\!\frac{H(S)}{\alpha\delta_N}=o(D_N/(C_N+1)),
\end{equation}
then both speciation endpoints and both nonlocality endpoints converge to $t=1$. In particular,
\begin{equation}\label{eq:separated-means-rate}
1-t_{\mathrm{spec}}^{\mathrm{start},\alpha\delta_N}
=O\!\left(\sqrt{\frac{\log[H(S)/(\alpha\delta_N)]}{D_N/(C_N+1)}}\right).
\end{equation}
\end{theorem}
\begin{proof}
First, we bound how often a noisy observation of $B$ is assigned the wrong label. In class $s$, the mean observation is $(1-t)m_s$. Use the rule that picks the closest class mean. It can mistake $s$ for $s'$ only if the observation crosses half the gap between these means.
 
 Let unit vector $e=(m_{s'}-m_s)/\|m_{s'}-m_s\|_2$ point toward that competing mean. Comparing the squared distances to $(1-t)m_s$ and $(1-t)m_{s'}$ shows that the displacement $Y=X_{B,t}-(1-t)m_s$ must satisfy
\begin{equation}
e^\top Y\geq z,\qquad z=\frac{1-t}{2}\|m_{s'}-m_s\|_2
\end{equation}
for an error to occur.

Conditional on $S=s$, we have $Y=(1-t)(X_{B,0}-m_s)+tZ_B$. To bound the chance of crossing the midpoint, we use a Chernoff bound on the conditional r.v. $e^\top Y \mid S$. In particular, we have
\begin{align}
    \Pr(\text{pairwise crossing}\mid S=s)&=\Pr(e^\top Y\geq z\mid S=s),\\ \nonumber
        &=\Pr(e^{\lambda e^\top Y}\geq e^{\lambda z}\mid S=s),\\ \nonumber
        &\leq e^{-\lambda z}\mathbb E[e^{\lambda e^\top Y}\mid S=s],\nonumber
\end{align}
where we take any $\lambda>0$, and the last step uses Markov's inequality.
 Plugging $u=\lambda(1-t)e$ into~\ref{eq:class-tail-bound} yields
 \begin{align}
     \mathbb E\!\left[e^{\lambda e^\top (1-t)(X_{B,0}-m_s)}\mid S=s\right]
&\leq e^{C_N\lambda^2(1-t)^2/2} .
 \end{align}
Together with \(\mathbb E\!\left[e^{\lambda e^\top tZ_{B}}\mid S=s\right]=e^{\lambda^2t^2/2}\) by Gaussianity, we obtain
\begin{equation}
 \Pr(\text{pairwise crossing}\mid S=s)
\leq\exp\!\left[-\lambda z+\frac{\lambda^2}{2}
\bigl(C_N(1-t)^2+t^2\bigr)\right].
\end{equation}
For $t<1$, right hand side is minimized at $\lambda=z/[C_N(1-t)^2+t^2]$, yielding $\exp[-z^2/(2(C(1-t)^2+t^2))]$. By~\ref{eq:mean-separation}, $z\geq(1-t)\sqrt{D_N}/2$; also $C_N(1-t)^2+t^2\leq C_N+1$. A union bound over the $L-1$ competing labels gives the following estimate for $t<1$:
\begin{equation}\label{eq:nearest-mean-error}
\Pr(\text{nearest-mean error})
\leq (L-1)e^{-D_N/(C_N+1)(1-t)^2/8}.
\end{equation}

Next, we connect classification error to the forward--backward success rate. Given $X_{B,t}$, write $q_s=p_t(s\mid X_{B,t})$. The best classifier chooses the label with largest $q_s$, so its error probability is $\mathbb E[1-\max_s q_s]$. Since $1-\sum_s q_s^2\leq2(1-\max_s q_s)$, equation~\ref{eq:fb-rate} and~\ref{eq:nearest-mean-error} give
\begin{equation}
1-P_{\mathrm{FB}}(B,t)
\leq 2\mathbb E[1-\max_s q_s]
\leq 2(L-1)e^{-D_N/(C_N+1)(1-t)^2/8}.
\end{equation}

The two-sided sandwich bounds in~\ref{eq:sandwich-1} now give $H(S\mid X_{B,t})\leq h_2(P_{\mathrm{FB}}(B,t))+(1-P_{\mathrm{FB}}(B,t))\log(L-1)$. For fixed $L$, the right-hand side is at most a constant times $\sqrt{1-P_{\mathrm{FB}}(B,t)}$. Combining this with the error bound above yields
\begin{equation}\label{eq:buffer-entropy-bound}
H(S\mid X_{B,t})\leq K e^{-kD_N/(C_N+1)(1-t)^2},
\end{equation}
where $K,k>0$ do not depend on $N$ or $t$.

Now set $t_N=1-b\sqrt{\log[H(S)/(\alpha\delta_N)]/(D_N/(C_N+1))}$, with $b$ a fixed large constant. The threshold condition makes $t_N\to1$, so $t_N\in[0,1]$ for large $N$. For every $t\leq t_N$, equation~\ref{eq:buffer-entropy-bound} is at most $K[\alpha\delta_N/H(S)]^{kb^2}$. Choose $b$ so that $kb^2>1$ and $K2^{-(kb^2-1)}<H(S)$. Since $\alpha\delta_N/H(S)<1/2$, this bound is smaller than $\alpha\delta_N$. Finally,
\begin{equation}
I_0(S;AB)-I_t(S;B)
\leq H(S)-I_t(S;B)
=H(S\mid X_{B,t}).
\end{equation}
Thus the speciation start is no earlier than $t_N$, which proves~\ref{eq:separated-means-rate}. Theorem~\ref{thm:window-containment} places the other three endpoints between that start and $1$. All four endpoints therefore converge to $1$.
\end{proof}

\end{document}